\documentclass[a4paper,fleqn]{cas-dc}

\usepackage[authoryear,longnamesfirst]{natbib}

\usepackage{amsmath,amsfonts,amssymb}
\usepackage{amsthm}
\usepackage{subcaption}
\usepackage{algorithm}
\usepackage{algpseudocode}

\theoremstyle{plain}
\newtheorem{proposition}{Proposition}

\providecommand{\linkable}[1]{\url{#1}}

\begin{document}
\let\WriteBookmarks\relax
\makeatletter\setlength{\@fptop}{0pt}\makeatother

\renewcommand{\topfraction}{0.95}
\renewcommand{\bottomfraction}{0.95}
\renewcommand{\textfraction}{0.05}
\renewcommand{\floatpagefraction}{0.75}
\renewcommand{\dbltopfraction}{0.95}
\renewcommand{\dblfloatpagefraction}{0.75}
\setcounter{topnumber}{4}
\setcounter{bottomnumber}{4}
\setcounter{totalnumber}{6}
\setcounter{dbltopnumber}{4}

\shorttitle{Diagnosing Shape-Prior Shortcuts in Long-Range Single-Shot FPP}
\shortauthors{A. Haroon et~al.}

\title[mode=title]{Diagnosing Shape-Prior Shortcuts in Long-Range Single-Shot Fringe Projection Profilometry}

\author[1]{Adam Haroon}
\ead{aharoon@iastate.edu}
\credit{Conceptualization, Methodology, Software, Validation, Formal Analysis, Investigation, Writing -- Original Draft, Writing -- Review \& Editing, Visualization, Project Administration}

\author[1]{Anush Lakshman}
\ead{anushl@iastate.edu}
\credit{Data Curation, Writing -- Review \& Editing}

\author[1]{Cody Fleming}
\ead{flemingc@iastate.edu}
\credit{Supervision, Resources, Writing -- Review \& Editing}

\author[2]{Beiwen Li}
\cormark[1]
\ead{beiwen.li@uga.edu}
\credit{Supervision, Writing -- Review \& Editing}

\affiliation[1]{organization={Department of Mechanical Engineering, Iowa State University},
            addressline={2529 Union Drive},
            city={Ames},
            postcode={50011},
            state={Iowa},
            country={USA}}

\affiliation[2]{organization={College of Engineering, University of Georgia},
            addressline={597 D. W. Brooks Drive},
            city={Athens},
            postcode={30602},
            state={Georgia},
            country={USA}}

\cortext[1]{Corresponding author}

\begin{abstract}
Learning-based single-shot fringe projection profilometry (FPP) has been studied almost entirely at close range, and the networks used are evaluated only on aggregate error, leaving open whether they recover depth from fringe phase or from object-level shape cues that correlate with depth. This paper diagnoses that question mechanistically in the long-range regime (standoff beyond 1~m). Using FPP-ML-Bench, an open photorealistic synthetic benchmark (15{,}600 fringe images, 50 objects at 1.5--2.1~m), we first formalize why the single-shot fringe-to-depth mapping is more severely ill-posed at long range: it is non-injective without fringe-order information, and the depth error from an incorrect fringe order grows as $Z^2$ in the working distance. Systematic ablations, extended with a multi-frame study, establish a best UNet baseline at 14.54~mm object mean absolute error (MAE), 18\% of the 80~mm object depth range, with only a 1.9$\times$ spread across four architectures, indicating a representational rather than a capacity-bound limit. A mechanistic interpretability study, the first applied to an FPP network, localizes the cause: linear probing shows edges are 2.82$\times$ more decodable than depth, Grad-CAM shows attention favoring boundaries over fringes by 1.28$\times$, and an in-range flat-plane test collapses a featureless plane to background depth despite valid fringes. The baseline solves the task via object-boundary shape priors rather than fringe-phase decoding. Because the shortcut is a hypothesis-space property, additional data or larger models will not remove it, motivating an architectural repair that removes the shape-prior solution by construction.
\end{abstract}



\maketitle


\section{Introduction}
\label{sec:intro}

Fringe projection profilometry (FPP) recovers dense, non-contact 3D surface measurements at sub-millimeter accuracy and is used across surface inspection~\citep{qian2021high,deng2016three}, robotic scanning~\citep{haroon2024autonomous,wang2024robotic}, and manufacturing process control~\citep{zhang2023machine,zhang2022systematic}. This accuracy derives from multi-step phase-shifting acquisition~\citep{zhang2016high,geng2011structured}, in which several phase-shifted fringe patterns are projected and analytically combined to recover depth. The same multi-step capture that delivers this fidelity requires the scene, the camera-projector geometry, and the illumination to remain stable across the full sequence, restricting traditional FPP to controlled tabletop setups. Advances in machine and deep learning have therefore motivated \emph{single-shot} reconstruction, in which depth is recovered from one fringe image rather than from a phase-shifted sequence~\citep{zuo2022deep,van2019deep,nguyen2020single,wang2021single,ikeda2025deep,li2025deep,zhu2022hformer,balasubramaniam2023single,wang2025end}, lifting this stability constraint and enabling real-time measurement under non-stationary conditions such as moving objects, perturbed sensor positions, or variable lighting. This active single-shot literature has, however, concentrated almost exclusively on the short-range regime below 1~m standoff, with endoscopic~\citep{zuo2025deep} and tabletop~\citep{wang2021single} setups being typical. Extending FPP to real-world deployments where non-stationarity is inherent, such as inspecting objects too large to bring to a scanner, scanning scenes that require wider fields of view, or measuring targets the sensor cannot physically approach, requires long-range capability.

Yet long range is the harder regime, by orders of magnitude. The close-range regime in which current methods operate spans a narrow range of geometries: endoscopic and microscopic systems below 100~mm standoff, tabletop systems at 300--800~mm, and high-speed dynamic measurements that exploit the same close-range geometry to track motion. Wang et al.'s graphics-augmented SSSR-FPP~\citep{wang2021single}, Nguyen et al.'s structured-light DCNN~\citep{nguyen2020single}, Van der Jeught and Dirckx's single-shot SLP~\citep{van2019deep}, Ikeda et al.'s color-fringe approach~\citep{ikeda2025deep}, Li et al.'s inner-shifting-phase encoding~\citep{li2025deep}, Zhu et al.'s Hformer~\citep{zhu2022hformer}, and Balasubramaniam and Li's data-storage measurement~\citep{balasubramaniam2023single} all target this close-range regime. At long range the physics is qualitatively different. Structured-light signal-to-noise ratio degrades as $1/r^2$, fringe contrast drops as the projected pattern spreads, and depth resolution at fixed baseline degrades quadratically with range. The same physics also produces a fringe-order ambiguity whose induced depth error scales as $Z^2$ in the working distance; Section~\ref{sec:singleshot_theory} formalizes both consequences and pinpoints why the single-shot fringe-to-depth mapping is more severely ill-posed at long range than at close range.

Ground truth at long range is itself unsolved. Acquiring accurate physical ground truth at standoff $>1$~m is hard for reasons independent of any learning method. Reference measurements at these distances typically require coordinate-measuring machines or terrestrial laser scanners with their own millimeter-scale uncertainties, and registering those measurements to the camera frame with sub-millimeter accuracy is nontrivial. The regime is data-starved as a direct consequence. A validated synthetic pipeline sidesteps this. VIRTUS-FPP~\citep{HaroonVIRTUS2025} demonstrates a photorealistic virtual FPP framework in NVIDIA Isaac Sim whose fringe contrast, phase-unwrapping behavior, and depth-reconstruction accuracy are validated against physical measurements; building the benchmark on this validated framework makes long-range ground truth tractable at the scale ML training requires. The novelty of the regime and the difficulty of physical ground truth are therefore coupled: the simulator is the methodological move that makes the regime accessible at all.

Beyond the long-range gap, single-shot ML FPP as a whole lacks the shared infrastructure that mainstream computer vision takes for granted~\citep{deng2009imagenet,lin2014microsoft}. Across the literature, normalization conventions~\citep{feng2021calibration,ikeda2025deep,nguyen2020single}, loss functions~\citep{wang2025end,ikeda2025deep,li2025deep,wang2021single,zhu2022hformer}, evaluation regions, and aggregate metrics all vary, and cross-study comparison is therefore unreliable. Prior synthetic-data efforts~\citep{zheng2020fringe,ueda2021fringe,zhang2023measurement} demonstrate photorealistic FPP simulation but do not release the resulting datasets, limiting their utility as shared benchmarks. The recent single-shot structured-light benchmark of Evans et al.~\citep{evans2024benchmark} provides an open real-data dataset, but on a calibration target with random surface variations rather than object-diverse self-occluding geometries. As a result, no open-source photorealistic synthetic benchmark for single-shot FPP is currently available, and none of the existing datasets cover object-diverse self-occluding geometries at long-range standoff. Explicit reporting conventions for normalization, loss, masking, and training-set composition are similarly absent.

Within the single-shot literature, methods are trained end-to-end with pixelwise regression and evaluated only on aggregate reconstruction error~\citep{zuo2022deep,ikeda2025deep,nguyen2020single,wang2021single}. This leaves open whether they recover depth from the phase structure of the fringe pattern or from object-level shape cues that happen to correlate with depth in the training distribution. This is the shortcut-learning hypothesis of Geirhos et al.~\citep{geirhos2020shortcut}, and in metrology it has direct consequences: a shape-prior network will not generalize to novel geometries and cannot substitute for phase-based reconstruction even when aggregate error looks competitive. The residual error single-shot networks hit at long range may reflect this representational issue rather than capacity or data scarcity, but the question has not been adjudicated mechanistically.

We address the interpretive gap through the first mechanistic diagnosis of a single-shot FPP network, built on FPP-ML-Bench~\citep{haroon2026fppml}, an open photorealistic synthetic benchmark for single-shot FPP constructed on the VIRTUS-FPP framework (15{,}600 fringe images, 50 objects at 1.5--2.1~m standoff) with a standardized object/background/overall evaluation protocol we show can change reported object-only error by nearly an order of magnitude. Relative to the benchmark as originally introduced~\citep{haroon2026fppml}, this paper adds three elements: a formal analysis of single-shot fringe-order ambiguity and its long-range depth-error scaling (Section~\ref{sec:singleshot_theory}), a multi-frame training ablation (Section~\ref{sec:multiframe_ablation}), and the first mechanistic interpretability study of an FPP network (Section~\ref{sec:interpretability}). Systematic ablations of normalization, background processing, multi-frame composition, loss, and architecture isolate the design choices that govern single-shot performance and establish a best UNet baseline whose 14.54~mm object MAE (18\% of the 80~mm object depth range), with only a 1.9$\times$ spread across four architectures, points to a representational rather than a model-capacity limit.

Three mechanistic probes (linear probing, Grad-CAM, and an in-range flat-plane out-of-distribution test), to our knowledge the first application of mechanistic interpretability to FPP, converge on the diagnosis: the best UNet baseline solves the task via object-boundary shape priors rather than via the phase structure of the fringe pattern. This is the shortcut-learning failure of Geirhos et al.~\citep{geirhos2020shortcut} instantiated in optical metrology, and Proposition~\ref{prop:noninjective} identifies the mechanism: the mapping the network is asked to learn does not uniquely exist without fringe-order information, and an optimizer trained on $\ell_1$/$\ell_2$ regression resolves that non-injectivity by collapsing onto whatever shape-prior surrogate the training distribution supplies. Additional data or larger models will not close the gap, because they do not change the hypothesis space the optimizer searches. Removing the shortcut therefore requires an architectural change that routes reconstruction through wrapped phase and fringe order rather than learning the fringe-to-depth mapping directly.

This paper makes the following contributions. First, we adopt and extend FPP-ML-Bench~\citep{haroon2026fppml}, an open photorealistic synthetic benchmark for single-shot FPP that covers object-diverse self-occluding geometries at long-range standoff (1.5--2.1~m) with a standardized object/background/overall evaluation protocol and explicit reporting conventions for normalization, loss, masking, and training-set composition. Second, we formalize the single-shot ambiguity, proving the fringe-to-depth mapping is non-injective without fringe-order information and showing the induced depth error scales as $Z^2$, which accounts on theoretical grounds for why single-shot FPP is more severely ill-posed at long range; to our knowledge this is the first such analysis of the long-range single-shot regime. Third, we isolate the design choices that govern single-shot FPP performance through systematic ablations of normalization, background processing, multi-frame composition, loss design, and architecture, and find that the residual error of the best baseline reflects a representational rather than a capacity-bound limit. Fourth, we diagnose that representational limit via the first application of mechanistic interpretability to FPP, with three complementary probes (linear probing, Grad-CAM, and an in-range flat-plane out-of-distribution test) converging on the finding that the best UNet baseline solves the task via shape priors rather than fringe-phase decoding. The diagnosis points to an architectural repair that removes the shape-prior solution from the hypothesis space by construction rather than discouraging it with a loss penalty.

Section~\ref{sec:virtual} reviews the FPP pipeline and describes the VIRTUS-FPP framework used for data acquisition. Section~\ref{sec:data} describes the dataset and acquisition protocol. Section~\ref{sec:benchmarking} formalizes the single-shot setting and its theoretical limits (Section~\ref{sec:singleshot_theory}) and establishes the single-shot baseline through systematic ablations of normalization (Section~\ref{sec:normalization_comparison}), background processing (Section~\ref{sec:background_ablation}), multi-frame composition (Section~\ref{sec:multiframe_ablation}), loss design (Section~\ref{sec:loss_comparison}), and architecture (Section~\ref{sec:models}). Section~\ref{sec:interpretability} runs the mechanistic diagnosis on the best baseline. Section~\ref{sec:conclusion} concludes.

\section{Fringe Projection Profilometry}
\label{sec:virtual}

This section reviews the classical FPP pipeline on which the rest of the paper builds (Section~\ref{sec:fpp_principles}) and describes the virtual implementation used for data acquisition (Sections~\ref{sec:system_config}--\ref{sec:virtual_calib}). The virtual implementation is based on the VIRTUS-FPP~\citep{HaroonVIRTUS2025} framework.

\subsection{Principles}
\label{sec:fpp_principles}

FPP recovers per-pixel depth by projecting structured fringe patterns onto a scene and decoding the phase modulation that surface geometry imposes on those patterns~\citep{zhang2016high,geng2011structured}. The classical pipeline decomposes into three deterministic stages: wrapped phase recovery, fringe-order recovery, and triangulation.

\paragraph{Wrapped phase recovery.} A sinusoidal fringe pattern projected through a calibrated projector produces camera intensity images of the form
\begin{equation}\label{eq:prelim_intensity}
I_n(x,y) = I'(x,y) + I''(x,y)\cos\!\left(\phi(x,y) + \delta_n\right),
\end{equation}
where $I'$ is the background intensity, $I''$ is the modulation amplitude, $\phi(x,y) \in [-\pi, \pi)$ is the wrapped phase carrying the depth information, and $\delta_n = 2\pi n / N$ is the known phase shift of the $n$-th of $N$ projected patterns. Acquiring the full $N$-step sequence and demodulating analytically gives the wrapped phase
\begin{equation}\label{eq:prelim_phase}
\phi(x,y) = -\tan^{-1}\!\left(\frac{\sum_{n=1}^{N} I_n \sin\delta_n}{\sum_{n=1}^{N} I_n \cos\delta_n}\right).
\end{equation}
$N$ typically ranges from 4 for fast acquisition to 18 for high noise immunity~\citep{zhang2016high}.

\paragraph{Fringe-order recovery.} The wrapped phase $\phi$ is ambiguous modulo $2\pi$ across multiple fringe periods. The absolute phase
\begin{equation}\label{eq:prelim_absphase}
\Phi(x,y) = \phi(x,y) + 2\pi k(x,y)
\end{equation}
requires the integer fringe order $k(x,y) \in \mathbb{Z}$ at every pixel. Classical FPP recovers $k$ by an auxiliary procedure such as temporal unwrapping with gray-code projection, in which a separate set of binary patterns is projected alongside the fringe sequence and decoded into the fringe index at every pixel~\citep{sansoni1999three}.

\paragraph{Triangulation.} With calibrated camera and projector pinhole models $M^c$ and $M^p$, the absolute phase at each camera pixel maps to a projector coordinate via $u^p = \Phi(x,y)\,T_p / (2\pi)$, where $T_p$ is the projector pixel period. The 3D point is recovered by triangulating the camera ray (from the camera pixel) and the projector ray (from the recovered projector coordinate)~\citep{zhang2010recent}.

Classical FPP achieves sub-millimeter accuracy because every stage of the pipeline is deterministic and well-posed: the $N$-step capture uniquely determines $\phi$, gray-code decoding uniquely determines $k$, and calibrated triangulation maps $\Phi$ to depth. The single-shot setting that motivates this paper replaces both analytical stages with a single learned mapping from one fringe image to depth; Section~\ref{sec:singleshot_theory} formalizes the ambiguity that this collapse introduces. Training and evaluating such a mapping at long range requires paired fringe-depth data at a scale physical acquisition cannot reliably provide, for the reasons given in Section~\ref{sec:intro}; we therefore instantiate the classical pipeline above in a calibrated virtual FPP system, the VIRTUS-FPP framework~\citep{HaroonVIRTUS2025}, and the remainder of this section describes its configuration and calibration.

\subsection{System Configuration}
\label{sec:system_config}

VIRTUS-FPP is built in NVIDIA Isaac Sim, integrating OptiX ray tracing for photorealistic rendering, PhysX for physics, and Universal Scene Description (USD) for 3D composition. The virtual system consists of a calibrated camera-projector pair
(Table~\ref{tab:system_params}). The camera uses Isaac Sim's pinhole
primitive ($960\times960$ resolution, 50~cm focal length), while the
projector is modeled using a rectangular light source
($0.625\text{ m} \times 0.5\text{ m}$, 40~nits) with texture projection.
The projector is positioned 0.1~m below and 0.125~m to the left of the
camera for optimal triangulation geometry.

\begin{table}[pos=tbp]
\caption{Virtual camera and projector system parameters.}
\label{tab:system_params}

\centering
\small
\resizebox{\ifdim\width>\columnwidth\columnwidth\else\width\fi}{!}{%
\begin{tabular}{@{}ll@{}}
\toprule
\textbf{Camera Parameters} & \textbf{Value} \\
\midrule
Focal length & 50 cm \\

Horizontal aperture & 20.9995 cm \\

Vertical aperture & 15.2908 cm \\

Resolution & $960 \times 960$ pixels \\
\midrule
\textbf{Projector Parameters} & \textbf{Value} \\

Intensity & 40 nits \\

Height & 0.625 m \\

Width & 0.5 m \\

Pattern resolution & $912 \times 1140$ pixels \\
\bottomrule
\end{tabular}}
\end{table}

VIRTUS-FPP's key innovation is projector modeling through the inverse
camera model:
\begin{equation}
\begin{bmatrix} X \\ Y \\ Z \end{bmatrix}
= (M_{\mathrm{ext}})^{-1}(M_{\mathrm{int}})^{-1}
\begin{bmatrix} u \\ v \\ 1 \end{bmatrix},
\end{equation}
enabling accurate dimensional correspondence of projected fringe patterns
at any distance without hardware constraints. All objects in FPP-ML-Bench
use consistent matte material properties (roughness=0.95, specular=0.15,
AO-to-diffuse=0.95) representative of typical structured light
scanning.~\citep{ma16155443,mi13101607}

The rendering pipeline uses OptiX path tracing with specific
configurations: disabled sampled direct lighting mode to prevent phase
map artifacts, and disabled shadows for clean fringe patterns. This
physics-based approach captures complex light transport including
multi-bounce illumination, surface reflectivity variations, and ambient
occlusion.

\subsection{Virtual Calibration}
\label{sec:virtual_calib}

VIRTUS-FPP performs complete virtual calibration using procedurally
generated $5\times9$ asymmetric circular boards (10~mm diameter, 20~mm
spacing). The system captures 18 calibration poses yielding 936
calibration images in approximately 5~minutes (10,530~images/hour). The
calibrated system achieves sub-pixel accuracy (stereo reprojection error:
0.055506 pixels, projector error: 0.048609 pixels).~\citep{HaroonVIRTUS2025}
The VIRTUS-FPP simulation setup is illustrated in Fig.~\ref{fig:simfpp}.

\begin{figure}[pos=tbp]
    \centering
    \includegraphics[width=\linewidth]{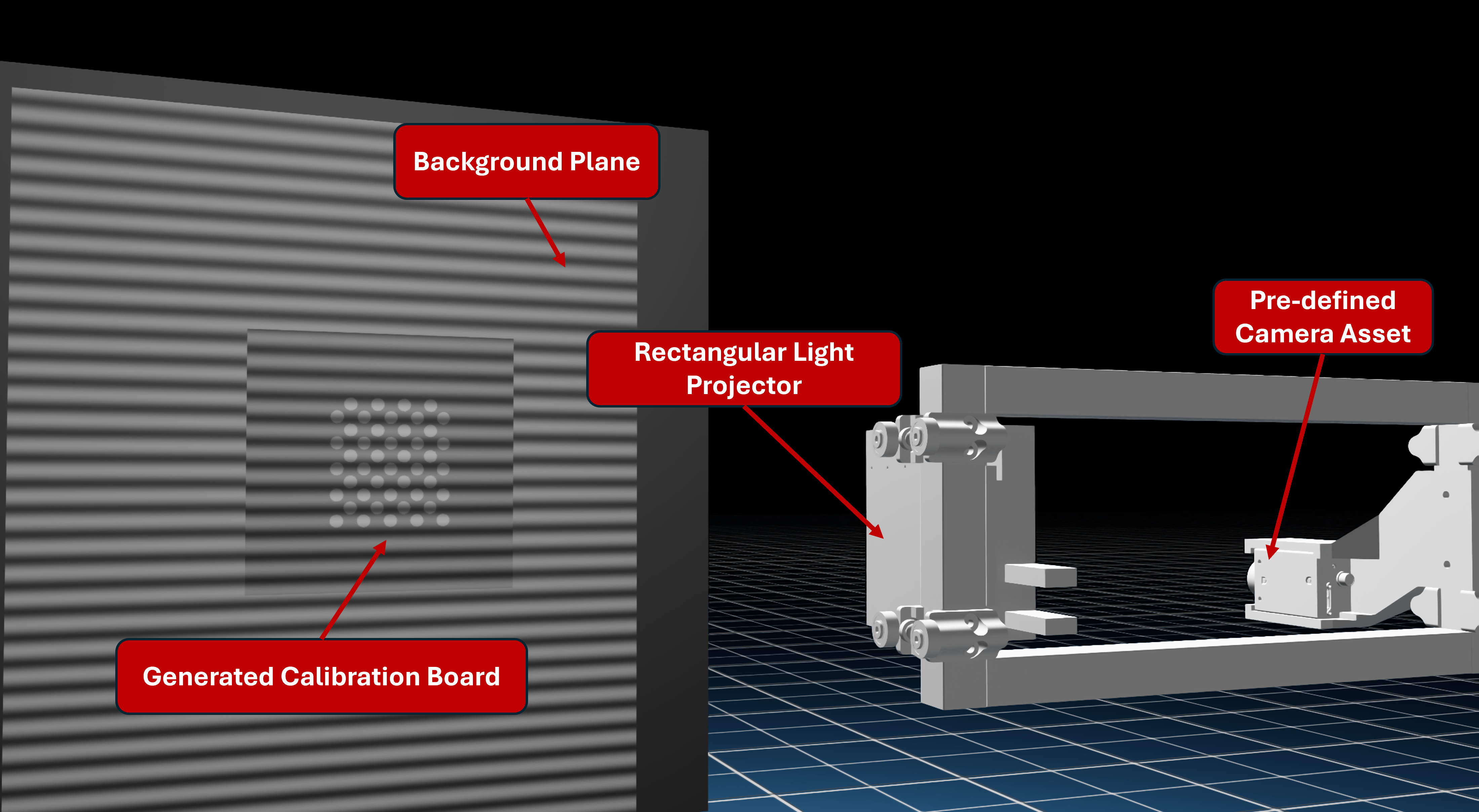}
    \caption{Virtual camera-projector calibration setup with a pinhole
    camera model, rectangular light-source projector, calibration board,
    and matte background plane.}
    \label{fig:simfpp}
\end{figure}

\section{Data Acquisition Methodology}
\label{sec:data}

\subsection{Dataset Composition}

We collected data for 50 USD objects from the YCB
dataset~\citep{calli2017yale} and NVIDIA Physical AI
Warehouse~\citep{Nvidia2025PhysicalAIWarehouse} spanning cylindrical
containers, rectangular boxes, complex shapes (power drills, spray guns),
and industrial components. This diversity evaluates robustness across
varying surface characteristics and morphological complexity, from simple
geometric primitives to intricate shapes with concavities and fine-scale
features.

The system was calibrated and operated within a range of 1.5--2.1~m,
with all objects positioned and scanned at distances within this range.
Objects were placed on a background plane with identical matte properties
to provide consistent lighting conditions and minimize unwanted
reflections. For multi-view acquisition, each object was rotated about
the vertical axis in $60^{\circ}$ increments using the rotation matrix:
\begin{equation}\label{eqn:rotationmatrix}
R_z(\theta_i) = \begin{bmatrix}
\cos\theta_i & -\sin\theta_i & 0 \\
\sin\theta_i & \cos\theta_i & 0 \\
0 & 0 & 1
\end{bmatrix}
\end{equation}
for $\theta_i = i \cdot 60^{\circ}$ where $i = 0,1,\ldots,5$. This
yields 6 viewpoints per object with approximately 50\% overlap between
adjacent views.

\subsection{Fringe Acquisition and Ground Truth Generation}

At each viewpoint, an 18-step phase-shifting sequence
($\delta_n = 2\pi n/18$, $n = 0,\ldots,17$) is captured at
$960\times960$ resolution:~\citep{zhang2016high}
\begin{equation}
I_n(u,v) = I'(u,v) + I''(u,v)\cos\!\left[\phi(u,v) + \frac{2\pi n}{18}\right],
\end{equation}
where $(u,v)$ are pixel coordinates, $I'(u,v)$ is the background
intensity, $I''(u,v)$ is the modulation amplitude, and $\phi(u,v)$ is
the wrapped phase. The GPU-accelerated pipeline achieves approximately
3~fps, over twice the speed of previous approaches.~\citep{zheng2020fringe}

The captured fringe patterns are processed using the standard 18-step
phase-shifting algorithm combined with Gray-code temporal
unwrapping~\citep{sansoni1999three} and triangulation to generate depth
maps $D(u,v)$ stored as MATLAB \texttt{.mat} files containing depth
values in millimeters at each pixel location.

\subsection{Dataset Summary}
\label{sec:dataset_summary}

The dataset comprises 15,600 raw fringe images (50 objects $\times$ 6
viewpoints $\times$ 52 fringe patterns per sequence), 300 corresponding
ground truth depth maps stored as \texttt{.mat} files, and 50 ground
truth mesh geometries. Data are partitioned with an 80/10/10 split at
the object level: 240 training samples (40 objects $\times$ 6
viewpoints), 30 validation samples (5 objects $\times$ 6 viewpoints),
and 30 test samples (5 objects $\times$ 6 viewpoints), ensuring
evaluation on completely unseen object geometries.

For all experiments, we use the first fringe image from each 18-step
sequence as the network input $I(u,v)$, ensuring identical conditions
across all models. To investigate how data representation affects
learning performance, we train models using three different normalization
strategies for the ground truth depth maps:

\begin{enumerate}
    \item \textbf{Raw depth (unnormalized):} The depth values $D(u,v)$
    are used directly in millimeters as stored in the \texttt{.mat}
    files, with values typically ranging from 0~mm (background) to
    1500--2000~mm (object surfaces). This approach requires the network
    to learn absolute metric depth values across a large dynamic range.

    \item \textbf{Global normalized depth:} The raw depth map is
    normalized by a global constant, converting millimeters to
    meters:~\citep{nguyen2020single}
    \begin{equation}
        D_{\mathrm{global}}(u,v) = \frac{D(u,v)}{1000}.
    \end{equation}
    This strategy reduces the numerical range to approximately 0--2~m,
    which may improve numerical stability during training while
    maintaining a consistent scale across all objects.

    \item \textbf{Individual normalized depth:} Each depth map is
    normalized independently to the range $[0,1]$ based on its
    object-specific depth range:
    \begin{equation}
        D_{\mathrm{norm}}(u,v) = \begin{cases}
            \dfrac{D(u,v) - D_{\min}}{D_{\max} - D_{\min}}
            & \text{if } D(u,v) > 0, \\[6pt]
            0 & \text{if } D(u,v) = 0,
        \end{cases}
    \end{equation}
    where $D_{\min} = \min_{u,v}\{D(u,v) \mid D(u,v)>0\}$ and
    $D_{\max} = \max_{u,v}\{D(u,v)\}$ are computed per object. The
    normalization parameters $(D_{\min}, D_{\max})$ are stored
    separately for each training sample to enable metric reconstruction
    during evaluation. This approach normalizes the learning problem to a
    consistent $[0,1]$ range across all training samples, effectively
    decoupling the learning of object \emph{shape} from absolute
    \emph{scale}.
\end{enumerate}

Fig.~\ref{fig:depth_normalizations} shows example depth visualizations
for each normalization strategy. While all three normalizations use the
same underlying \texttt{.mat} depth files for training, we visualize
them by converting to uint16 format via min-max scaling solely for
display purposes. The actual training data remain as floating-point depth
values.

\begin{figure}[pos=tbp]
    \centering
    \includegraphics[width=0.32\linewidth]{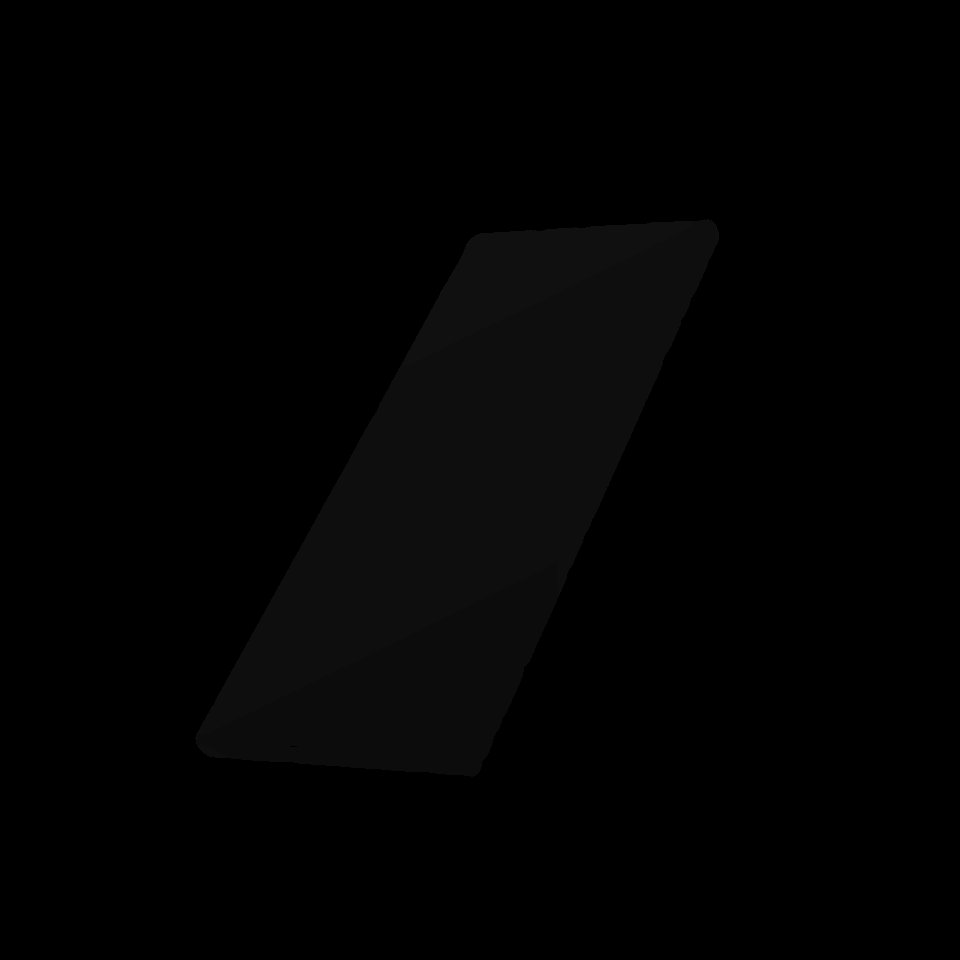}
    \includegraphics[width=0.32\linewidth]{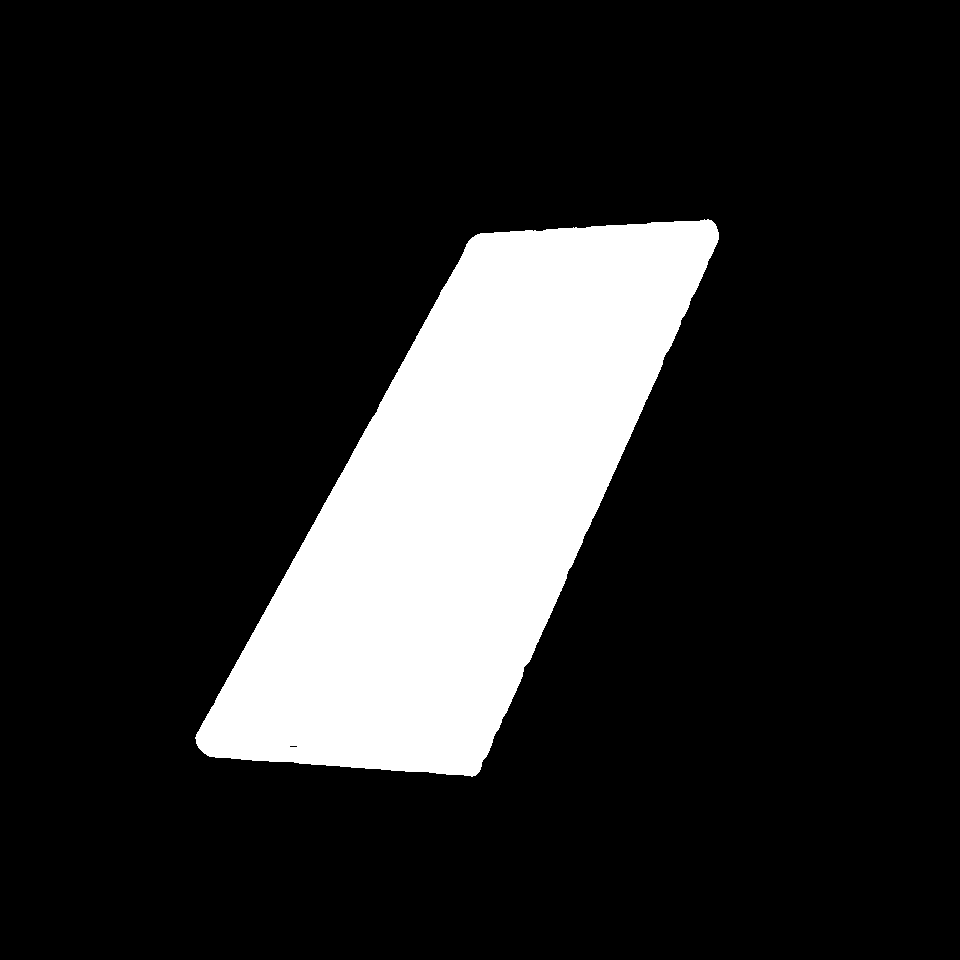}
    \includegraphics[width=0.32\linewidth]{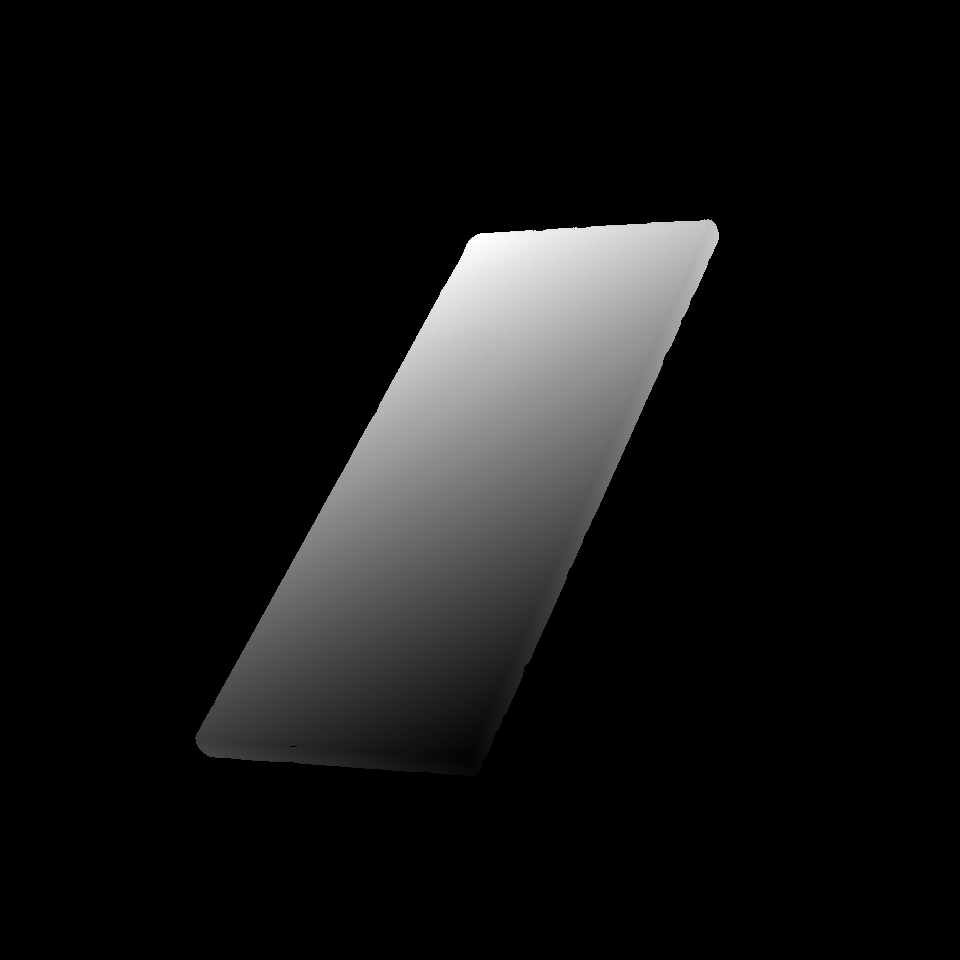}
    \caption{Depth map visualizations for three normalization strategies
    on the same object (wooden boards). From left to right: raw depth
    (0--2026~mm), global normalized depth (0--2.026~m), and individual
    normalized depth ($[0,1]$ range mapped to 1561--2026~mm). All use the
    same underlying depth data, differing only in normalization.}
    \label{fig:depth_normalizations}
\end{figure}

\section{Single-Shot 3D Reconstruction Benchmarking}
\label{sec:benchmarking}

\subsection{Problem Formulation}
\label{sec:problem_formulation}

The classical FPP pipeline reviewed in Section~\ref{sec:fpp_principles} acquires $N$ phase-shifted fringe images to recover the wrapped phase analytically and a separate gray-code sequence to recover the integer fringe order, then triangulates to depth. Single-shot reconstruction collapses this pipeline to a single fringe image $I$ and replaces both analytical stages with a learned mapping $\hat{D}_{\mathrm{norm}} = f_\theta(I)$, where $f_\theta$ is a neural network. We use the first fringe from each 18-step sequence as input, ensuring identical conditions across models.

This collapsed formulation is inherently ill-posed. Single fringe images contain ambiguity in depth estimation owing to the periodic nature of sinusoidal patterns: in the absence of phase-shifting redundancy or auxiliary unwrapping patterns, each $2\pi$ phase range looks identical to every other, and the fringe-order ambiguity that classical gray-code decoding resolves analytically is no longer separable from phase prediction. As a result, learning-based approaches rely on inferring depth from learned shape priors and statistical regularities rather than from fully explicit geometric cues alone. Section~\ref{sec:singleshot_theory} formalizes this ambiguity and shows that the depth error induced by an incorrect fringe-order estimate scales quadratically with the working distance, accounting on theoretical grounds for why single-shot FPP is particularly difficult at long range.

We ablate the five design choices that govern single-shot FPP performance: depth normalization, background processing, multi-frame composition, loss function, and architecture. This isolates the impact of each and establishes best practices for learning-based single-shot FPP.

\subsection{Theoretical Analysis: Single-Shot Ambiguity and Long-Range Sensitivity}
\label{sec:singleshot_theory}

We formalize the ambiguity induced by using a single fringe image for depth reconstruction. Consider a projector-camera system in which a sinusoidal fringe pattern is projected onto a scene and observed by the camera. The captured image intensity at camera pixel $(x,y)$ for the $t$-th phase-shifted pattern can be modeled as
\begin{equation}\label{eq:intensity_model}
I_t(x,y) = A(x,y) + B(x,y)\cos\!\left(\Phi(x,y) + \delta_t\right),
\end{equation}
where $A(x,y)$ is the background illumination, $B(x,y)$ is the fringe modulation, $\delta_t$ is the known phase shift, and $\Phi(x,y)$ is the absolute phase associated with the projector-camera correspondence. The absolute phase decomposes as
\begin{equation}\label{eq:absolute_phase}
\Phi(x,y) = \phi(x,y) + 2\pi k(x,y),
\end{equation}
where $\phi(x,y) \in [-\pi,\pi)$ is the wrapped phase and $k(x,y) \in \mathbb{Z}$ is the fringe order. The wrapped phase describes the local phase within one fringe period, whereas the fringe order identifies the global fringe period to which the observed point belongs.

\begin{proposition}\label{prop:noninjective}
In single-shot sinusoidal fringe projection profilometry, the mapping from a single captured fringe image $I_t(x,y)$ to scene depth $Z(x,y)$ is not uniquely determined unless the fringe order $k(x,y)$ is known or otherwise constrained.
\end{proposition}

\begin{proof}
From Eq.~\eqref{eq:intensity_model}, the observed intensity depends on the absolute phase through a cosine function. Since the cosine function is periodic with period $2\pi$, for any integer $m \in \mathbb{Z}$,
\begin{equation}\label{eq:periodicity}
\cos\!\left(\Phi(x,y) + \delta_t\right) = \cos\!\left(\Phi(x,y) + 2\pi m + \delta_t\right).
\end{equation}
Therefore, two absolute phases
\begin{equation}
\Phi_1(x,y) = \phi(x,y) + 2\pi k(x,y)
\end{equation}
and
\begin{equation}
\Phi_2(x,y) = \phi(x,y) + 2\pi\left(k(x,y)+m\right), \qquad m \neq 0,
\end{equation}
produce identical observed intensities:
\begin{equation}\label{eq:same_intensity}
I_t\!\left(\Phi_1(x,y)\right) = I_t\!\left(\Phi_2(x,y)\right).
\end{equation}

However, these two absolute phases correspond to different projector coordinates. If the projected fringe pitch is $P$ pixels, the projector coordinate associated with the absolute phase is
\begin{equation}\label{eq:projector_coordinate}
u_p(x,y) = \frac{P}{2\pi}\,\Phi(x,y).
\end{equation}
Substituting Eq.~\eqref{eq:absolute_phase} into Eq.~\eqref{eq:projector_coordinate} gives
\begin{equation}\label{eq:projector_coordinate_order}
u_p(x,y) = P\!\left(\frac{\phi(x,y)}{2\pi} + k(x,y)\right).
\end{equation}
Thus, changing the fringe order by $m$ changes the projector coordinate by
\begin{equation}\label{eq:projector_error}
\Delta u_p = mP.
\end{equation}
Although the observed intensity is unchanged under this shift, the projector-camera correspondence is not. Since depth reconstruction in FPP is obtained through triangulation from the camera pixel and the corresponding projector coordinate, different values of $u_p$ generally yield different depths. The same single-shot intensity observation is therefore consistent with multiple physically distinct depths. Hence, the mapping
\begin{equation}
I_t(x,y) \mapsto Z(x,y)
\end{equation}
is non-injective in the absence of fringe-order information.
\end{proof}

\subsubsection{Effect of Range on Depth Error}

We next show that an incorrect fringe-order estimate produces increasingly large depth errors as the object distance increases. For clarity, consider a simplified triangulation model analogous to stereo geometry:
\begin{equation}\label{eq:triangulation}
Z = \frac{fb}{d},
\end{equation}
where $Z$ is depth, $f$ is the focal length, $b$ is the projector-camera baseline, and $d$ is the disparity induced by the projector-camera correspondence. A perturbation in correspondence produces a depth perturbation
\begin{equation}
\Delta Z \approx \frac{\partial Z}{\partial d}\,\Delta d.
\end{equation}
Differentiating Eq.~\eqref{eq:triangulation} with respect to $d$ gives
\begin{equation}\label{eq:depth_derivative}
\frac{\partial Z}{\partial d} = -\frac{fb}{d^2}.
\end{equation}
Since $d = fb/Z$, Eq.~\eqref{eq:depth_derivative} can be rewritten as
\begin{equation}\label{eq:depth_error}
\left|\Delta Z\right| \approx \frac{Z^2}{fb}\left|\Delta d\right|.
\end{equation}
A fringe-order error of $m$ periods induces a projector-coordinate error of approximately $mP$ and therefore a correspondence error proportional to $mP$:
\begin{equation}\label{eq:range_error}
\left|\Delta Z\right| \propto \frac{Z^2}{fb}\,|mP|.
\end{equation}
Equation~\eqref{eq:range_error} shows that the depth error caused by an incorrect fringe order grows approximately quadratically with the object distance $Z$. While fringe-order ambiguity exists at all working distances, its effect becomes significantly more severe in long-range FPP, explaining why single-shot fringe-to-depth learning becomes increasingly ill-posed at distances greater than approximately one meter, where a small correspondence ambiguity can lead to a large depth error. To our knowledge, this is the first analysis to characterize how fringe-order ambiguity governs depth error specifically in the long-range single-shot FPP regime, which the existing single-shot literature has not addressed.

\subsubsection{Implication for Learning-Based Single-Shot FPP}

The above result has a direct implication for learning-based single-shot FPP. If a neural network is trained to directly learn
\begin{equation}
f_{\theta}: I_t(x,y) \mapsto Z(x,y),
\end{equation}
then it is asked to approximate a mapping that may not be uniquely defined. In particular, two points may have identical or nearly identical local fringe observations,
\begin{equation}
I_t^{(1)}(x,y) \approx I_t^{(2)}(x,y),
\end{equation}
while corresponding to different depths,
\begin{equation}
Z^{(1)}(x,y) \neq Z^{(2)}(x,y),
\end{equation}
because they belong to different fringe orders. Under standard regression losses such as $\ell_1$ or $\ell_2$ loss, the network may learn an averaged or biased depth estimate rather than the physically correct depth.

The implication is that, for long-range single-shot FPP, the network must be provided with additional information that resolves the phase ambiguity. This can take several forms: explicit supervision of the fringe order $k(x,y)$, recovery of the absolute phase $\Phi(x,y)$, spatial or temporal coding constraints, or an auxiliary phase-unwrapping objective. In any of these formulations, the reconstruction pipeline becomes
\begin{equation}
I_t(x,y) \mapsto \left(\phi(x,y), k(x,y)\right) \mapsto \Phi(x,y) \mapsto Z(x,y),
\end{equation}
rather than directly learning an ill-posed mapping from a single fringe image to depth. An architecture that realizes this pipeline, by making the network's output space wrapped phase and inserting a fixed calibration layer that maps phase to depth deterministically, is the architectural direction this diagnosis points to.

The empirical predictions of this analysis are confirmed in the remainder of the paper. The architectural baselines of Section~\ref{sec:models} all plateau in the regime that Proposition~\ref{prop:noninjective} predicts is irreducibly ill-posed without fringe-order information, with errors that the mechanistic analysis of Section~\ref{sec:interpretability} traces directly to the network's inability to recover phase.

\subsection{Normalization Strategy Comparison}
\label{sec:normalization_comparison}

Having formalized the limit set by single-shot fringe-order ambiguity, we now turn to the empirical question of how well a learning-based mapping can be made to perform within that limit. We begin by isolating data normalization, which we will see is the single most impactful design choice for an end-to-end depth-regression baseline; the remaining subsections then sweep background processing, multi-frame composition, loss design, and architecture in turn.

\subsubsection{Architecture and training configuration}

We use UNet~\citep{Ronneberger2015} as our baseline architecture due to its
widespread adoption in FPP
literature~\citep{feng2021calibration,zheng2020fringe,ikeda2025deep} and
strong performance on dense prediction tasks. The network consists of
four encoder-decoder stages, progressively downsampling from
$960\times960$ to $60\times60$ resolution, with channel depths increasing
from 64 to 1024 at the bottleneck. Each encoder and decoder stage uses
two $3\times3$ convolutions followed by instance normalization and ReLU
activation. Skip connections preserve spatial information across
corresponding encoder-decoder levels.

We train with RMSE loss using the RMSprop optimizer ($\alpha=0.99$,
initial learning rate $10^{-4}$, weight decay $10^{-5}$), batch size 4,
and a ReduceLROnPlateau scheduler (factor=0.5, patience=10 epochs,
minimum learning rate $10^{-6}$). This combination has demonstrated
effectiveness for similar dense regression tasks~\citep{ikeda2025deep} and
provides a stable baseline for comparison. Training continues until the
learning rate reaches its minimum threshold, indicating convergence.

\subsubsection{Evaluation metrics}

Because background pixels (which carry no valid depth and should predict
0~mm) occupy 60--90\% of each image, they dominate aggregate error and
dilute object error, so a single overall metric is misleading. We
therefore report error on three pixel sets: the full image grid, the object set
$\mathcal{M}^{+}=\{(u,v):D(u,v)>0\}$ (using $D_{\mathrm{norm}}(u,v)>0$
under individual normalization), and the background set
$\mathcal{M}^{0}=\{(u,v):D(u,v)=0\}$, with $|\cdot|$ denoting cardinality.
For a pixel set $\mathcal{S}$,
\begin{equation}
    \mathrm{MAE}_{\mathcal{S}} = \frac{1}{|\mathcal{S}|}
    \sum_{(u,v)\in\mathcal{S}}\bigl|\hat{D}-D\bigr|,
    \qquad
    \mathrm{RMSE}_{\mathcal{S}} = \sqrt{\frac{1}{|\mathcal{S}|}
    \sum_{(u,v)\in\mathcal{S}}\bigl(\hat{D}-D\bigr)^{2}},
\end{equation}
and we write the three instantiations with subscripts $\mathrm{ovr}$
(full grid), $\mathrm{obj}$ ($\mathcal{M}^{+}$), and $\mathrm{bg}$
($\mathcal{M}^{0}$). The object metric isolates accuracy on the actual
geometry, the background metric measures suppression, and the overall
metric is the holistic system-level value.

All errors are reported in millimeters after denormalization to enable
direct comparison across normalization strategies. For individual
normalized predictions, we apply the inverse transformation
$\hat{D}(u,v) = \hat{D}_{\mathrm{norm}}(u,v)\cdot(D_{\max}-D_{\min})+D_{\min}$
using the stored normalization parameters.

\subsubsection{Results}

Table~\ref{tab:normalization_results} summarizes the quantitative
results across the 30 test samples for each normalization strategy, and
Fig.~\ref{fig:normalization_predictions} visualizes detailed per-sample
predictions for a representative object (magazine stack) under each
normalization strategy.

\begin{table}[pos=tbp]
\caption{Reconstruction error for three normalization strategies on 30
test samples. All errors reported in millimeters.}
\label{tab:normalization_results}

\centering
\small
\resizebox{\ifdim\width>\columnwidth\columnwidth\else\width\fi}{!}{%
\begin{tabular}{@{}lcccccc@{}}
\toprule

\textbf{Normalization} &
\textbf{Overall} & \textbf{Overall} &
\textbf{Object} & \textbf{Object} &
\textbf{Background} & \textbf{Background} \\

 & \textbf{MAE} & \textbf{RMSE} &
   \textbf{MAE} & \textbf{RMSE} &
   \textbf{MAE} & \textbf{RMSE} \\
\midrule
Raw        & 35.20 &  85.64 & 148.07 & 214.37 & 19.41 & 52.38 \\

Global     & 14.81 &  57.66 &  82.49 & 144.23 &  9.90 & 45.69 \\

Individual & \textbf{2.30} & \textbf{6.80} & \textbf{16.20} & \textbf{21.19} & \textbf{0.92} & \textbf{3.00} \\
\bottomrule
\end{tabular}}
\end{table}

\begin{figure*}[pos=tp]
    \centering
    \includegraphics[width=0.95\linewidth]{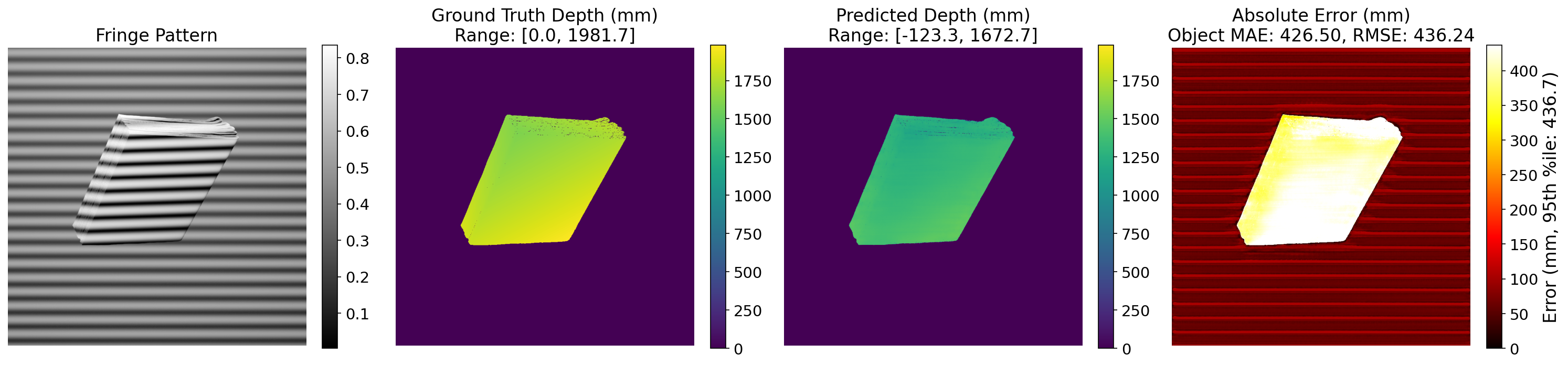}
    \vspace{0.2cm}
    \includegraphics[width=0.95\linewidth]{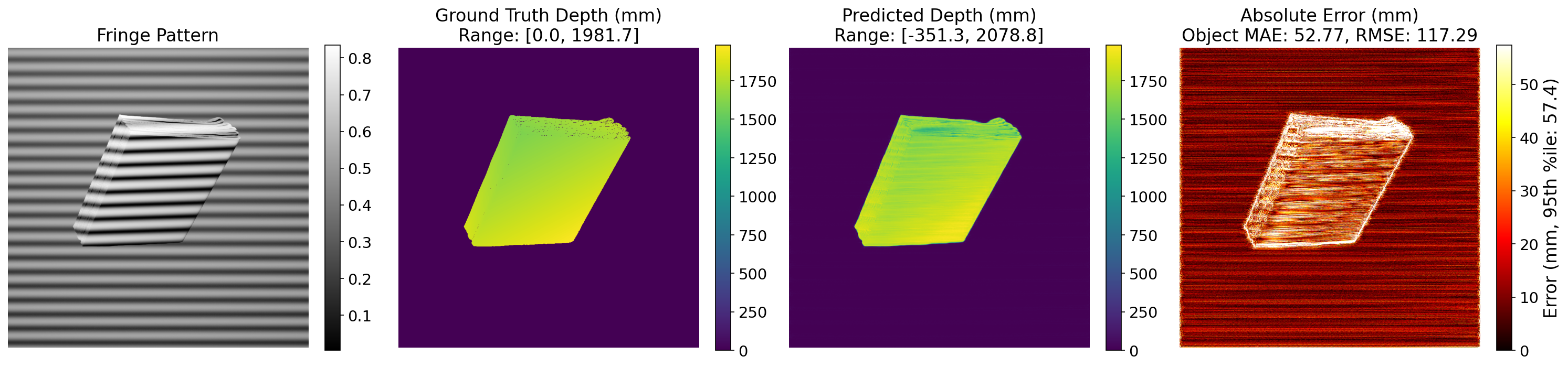}\\
    \vspace{0.2cm}
    \includegraphics[width=0.95\linewidth]{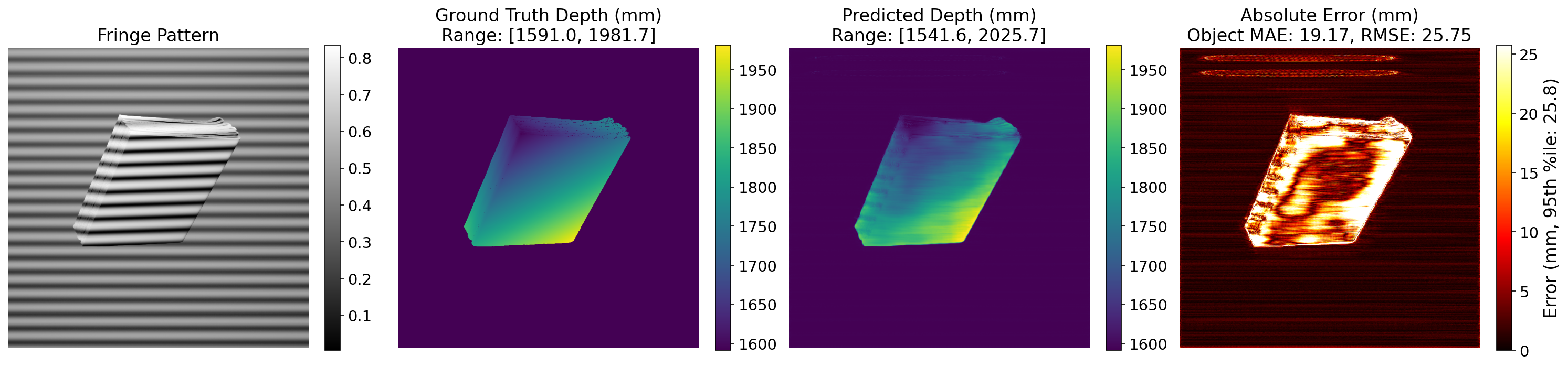}
    \caption{Single-shot depth reconstruction for the magazine stack
    object comparing three normalization strategies. From top to bottom:
    raw depth (426.50~mm object MAE), global normalized depth (52.77~mm
    object MAE), individual normalized depth (19.17~mm object MAE). Each
    row shows, from left to right: input fringe pattern, ground truth
    depth, predicted depth, and absolute error map (clipped at 95th
    percentile for visibility). Individual normalization achieves
    22$\times$ improvement over raw and 2.7$\times$ improvement over
    global.}
    \label{fig:normalization_predictions}
\end{figure*}

Individual normalization substantially outperforms both alternatives
across all metrics. Most notably, the object-only MAE improves from
148.07~mm (raw) to 82.49~mm (global) to just 16.20~mm (individual),
representing a 9.1$\times$ improvement over raw depth and a 5.1$\times$
improvement over global normalization. Background suppression also
improves dramatically, with individual normalization achieving near-zero
background error (0.92~mm MAE, 3.00~mm RMSE) compared to 19.41~mm MAE
for raw depth.

The overall error statistics can be misleading without this
decomposition: raw depth shows 35.20~mm overall MAE despite catastrophic
148.07~mm object error because the numerically larger number of
background pixels (with lower 19.41~mm error) dilutes the overall
statistic. This highlights the importance of reporting object-specific
metrics for FPP evaluation.

The performance ranking directly correlates with optimization difficulty.
Raw depth requires the network to learn absolute metric values across a
0--2000~mm range while simultaneously learning object geometry, creating
a challenging multi-scale optimization problem. Global normalization
improves this by reducing the range to 0--2~m, but different objects
still exhibit different depth ranges in this normalized space, requiring
the network to handle scale variation. Individual normalization
fundamentally simplifies the learning problem by decoupling scale from
shape. In the $[0,1]$ normalized space, the network need only learn the
relative depth structure of each object, with absolute scale provided by
the stored normalization parameters. This consistent range across all
training samples creates a more uniform optimization landscape.
Additionally, background pixels map to 0.0 in normalized space for all
objects, becoming trivial to learn and resulting in near-perfect
background suppression that persists after denormalization.

\subsubsection{Background fringe ablation study}
\label{sec:background_ablation}

\begin{table}[pos=tbp]
\caption{Impact of background fringe removal on reconstruction error
(millimeters). All three normalization strategies show severe performance
degradation when background fringes are removed from input images.}
\label{tab:background_ablation}

\centering
\small
\resizebox{\ifdim\width>\columnwidth\columnwidth\else\width\fi}{!}{%
\begin{tabular}{@{}lcccccc@{}}
\toprule

\textbf{Normalization} &
\textbf{Overall} & \textbf{Overall} &
\textbf{Object} & \textbf{Object} &
\textbf{Background} & \textbf{Background} \\

 & \textbf{MAE} & \textbf{RMSE} &
   \textbf{MAE} & \textbf{RMSE} &
   \textbf{MAE} & \textbf{RMSE} \\
\midrule
Raw        & 112.54 & 189.70 & 437.40 & 511.08 &  91.99 & 178.83 \\

Global     & 115.41 & 191.63 & 598.40 & 723.93 &  92.26 & 151.88 \\

Individual &  15.00 &  21.69 &  45.01 &  57.07 &  12.16 &  15.60 \\
\bottomrule
\end{tabular}}
\end{table}

\begin{figure*}[pos=tp]
    \centering
    \includegraphics[width=0.95\linewidth]{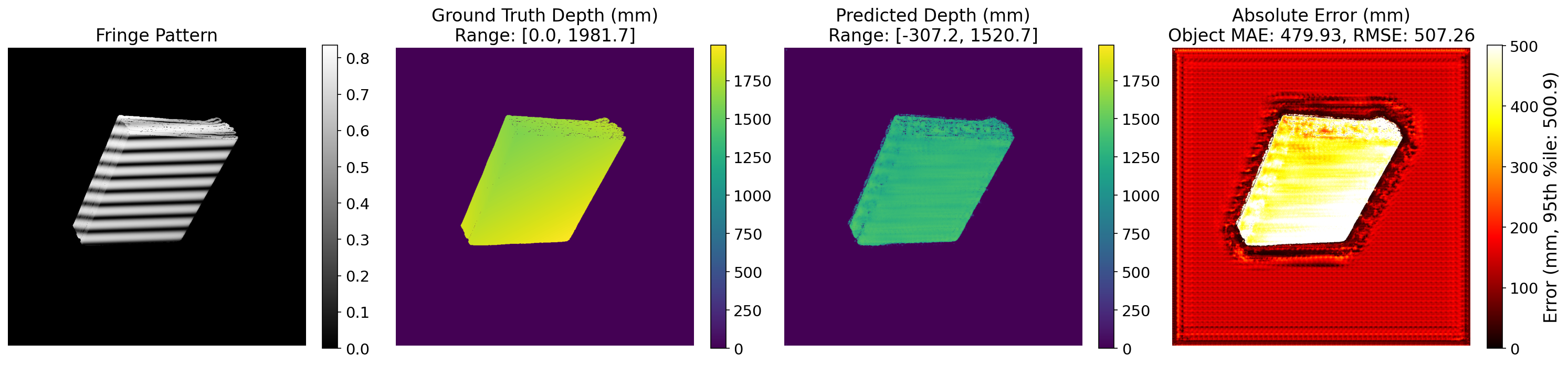}
    \vspace{0.2cm}
    \includegraphics[width=0.95\linewidth]{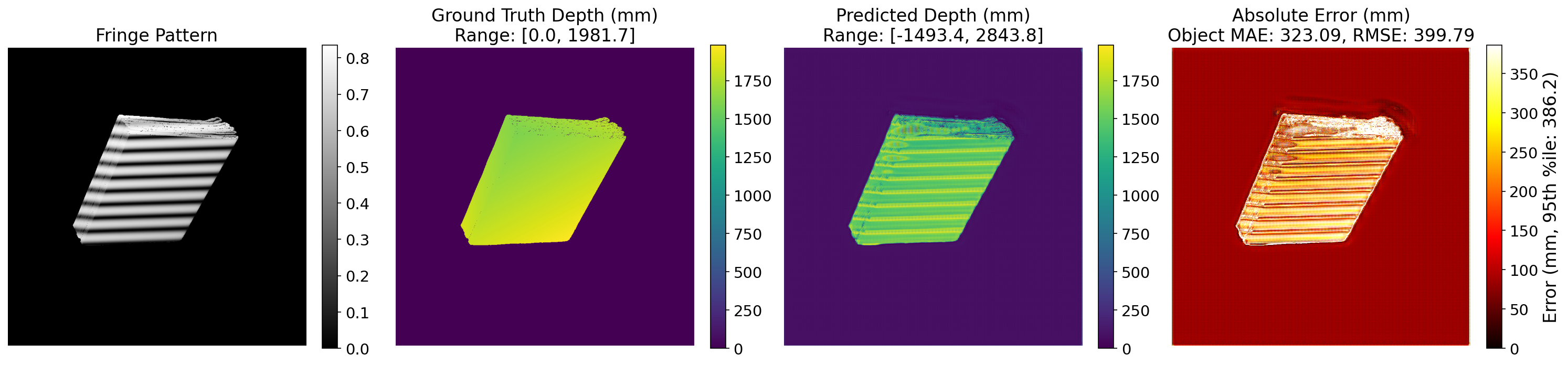}\\
    \vspace{0.2cm}
    \includegraphics[width=0.95\linewidth]{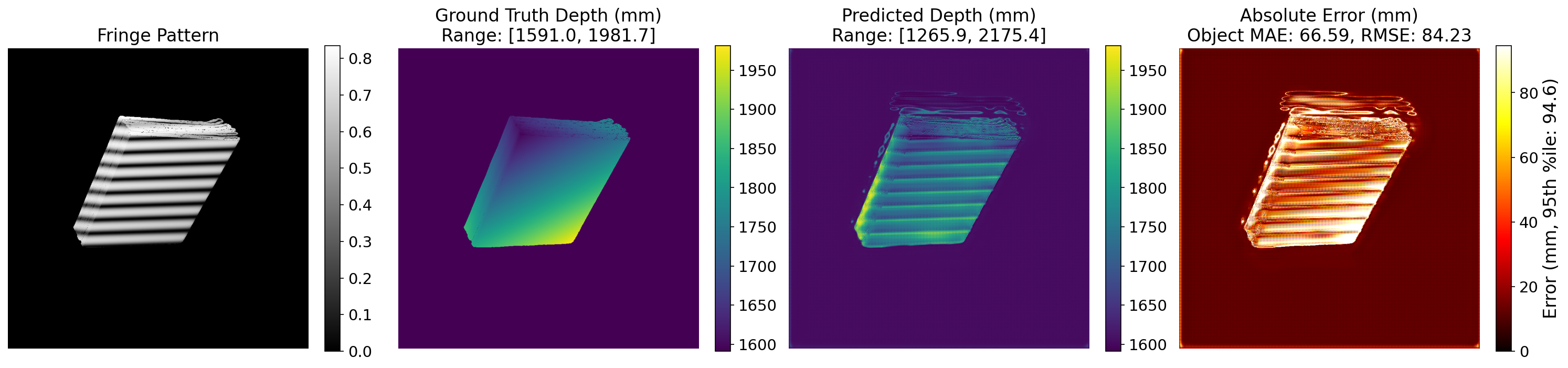}
    \caption{Impact of background fringe removal across three
    normalization strategies for the same magazine stack object. From top
    to bottom: raw depth (479.93~mm object MAE), global normalized depth
    (323.03~mm object MAE), individual normalized depth (66.59~mm object
    MAE). Background removal degrades all strategies: raw 1.1$\times$,
    global 6.1$\times$, individual 3.5$\times$ worse. Note severe
    boundary artifacts and inconsistent depth prediction compared with
    Fig.~\ref{fig:normalization_predictions}.}
    \label{fig:background_ablation_predictions}
\end{figure*}

Visual inspection of prediction error maps reveals horizontal fringe
artifacts in background regions of predicted depth, suggesting the
network attempts to predict depth from background fringe patterns. To
investigate whether these background fringes help or hinder learning, we
conduct an ablation study by removing them from the input.

\noindent\textbf{Experimental setup.} We mask the input fringe images by
setting all background pixels (where $D(u,v)=0$) to zero intensity,
effectively removing the background fringe pattern. We retrain the same
UNet architecture with identical hyperparameters on all three
normalization strategies using these background-removed inputs, then
evaluate on test data with similarly masked inputs.

\noindent\textbf{Results.} Table~\ref{tab:background_ablation} shows
that, surprisingly,
background removal \emph{catastrophically degrades} performance across
all normalization strategies. For individual normalization, object MAE
increases from 16.20~mm to 45.01~mm (2.8$\times$ worse), while
background error increases from 0.92~mm to 12.16~mm (13.2$\times$
worse). Raw and global normalizations suffer even more severe
degradation, with object errors increasing by 3.0$\times$ and
7.3$\times$, respectively. Fig.~\ref{fig:background_ablation_predictions}
shows that predictions with masked inputs exhibit severe boundary
artifacts and depth discontinuities not present with full fringe
patterns.

This counterintuitive result reveals that background fringes provide
essential spatial context rather than acting as distractors. The
continuous fringe pattern across the entire image enables the network to
establish a spatial phase reference that helps resolve the inherent
$2\pi$ ambiguity of single-shot reconstruction. Natural phase
discontinuities at object boundaries provide clearer edge information
than the artificial black-to-object transitions created by masking,
allowing the network to learn where fringe patterns abruptly change phase
as indicators of depth transitions. Additionally, background pixels
provide training signal and implicit regularization, helping the network
understand the overall fringe-to-depth relationship and preventing
overfitting to object-only patterns. This finding has important
implications for FPP learning: background fringes are \emph{signal}, not
noise. All subsequent experiments therefore use full fringe patterns with
background intact.

\subsection{Multi-Frame Training Ablation}
\label{sec:multiframe_ablation}

Having established individual normalization and background retention as
optimal, we investigate whether expanding the training set using
additional fringe frames from each acquisition sequence improves
reconstruction accuracy. In FPP-ML-Bench, each object-viewpoint
combination yields 52 captured images: 18~horizontal sinusoidal fringe
patterns (varying phase), 18~vertical sinusoidal fringe patterns,
7~horizontal gray code patterns, 7~vertical gray code patterns, one
black image, and one white image. All 52~images share the same ground
truth depth map. The baseline experiments above use only the first
horizontal fringe image (frame~0) as input, yielding 240~training
samples. Here, we evaluate whether pairing additional frames with the
same depth targets, effectively treating each frame as an independent
training sample, improves generalization.

\subsubsection{Experimental setup}

We train the same UNet architecture with identical hyperparameters under
four multi-frame conditions:

\begin{enumerate}
    \item \textbf{Horizontal only:} 18 horizontal sinusoidal frames per
    sequence (4{,}320 training / 540 test samples).
    \item \textbf{Vertical only:} 18 vertical sinusoidal frames per
    sequence (4{,}320 training / 540 test samples).
    \item \textbf{All sinusoidal:} All 36 sinusoidal frames (horizontal
    + vertical) per sequence (8{,}640 training / 1{,}080 test samples).
    \item \textbf{All frames:} All 52 frames including gray code, black,
    and white patterns (12{,}480 training / 1{,}560 test samples).
\end{enumerate}

Each condition is evaluated on its corresponding test set, where each
test sample pairs a fringe frame with the same ground truth depth map
used in the baseline.

\subsubsection{Results}

Table~\ref{tab:multiframe_results} presents the quantitative results,
and Figs.~\ref{fig:multiframe_histograms_single}--\ref{fig:multiframe_histograms_mixed}
show the error distributions for each condition.

\begin{table}[pos=tbp]
\caption{Multi-frame training ablation results. The baseline uses a
single horizontal fringe frame per sequence (240~training samples). Each
multi-frame condition expands training by pairing additional fringe
frames with the same ground truth depth map. All errors in millimeters.}
\label{tab:multiframe_results}

\centering
\small
\resizebox{\ifdim\width>\columnwidth\columnwidth\else\width\fi}{!}{%
\begin{tabular}{@{}lcccccccc@{}}
\toprule

\textbf{Training} & \textbf{Train} & \textbf{Test} &
\textbf{Overall} & \textbf{Overall} &
\textbf{Object} & \textbf{Object} &
\textbf{BG} & \textbf{BG} \\

\textbf{Condition} & \textbf{Samples} & \textbf{Samples} &
\textbf{MAE} & \textbf{RMSE} &
\textbf{MAE} & \textbf{RMSE} &
\textbf{MAE} & \textbf{RMSE} \\
\midrule
Baseline (frame 0)        & 240      & 30      & 2.30 & 6.80  & 16.20 & 21.19 & 0.92 & 3.00 \\

Horizontal only            & 4{,}320  & 540     & \textbf{1.40} & \textbf{5.54}  & 12.87 & \textbf{16.89} & 0.21 & \textbf{2.26} \\

Vertical only              & 4{,}320  & 540     & 1.26 & 5.54  & \textbf{12.36} & 16.94 & \textbf{0.16} & 2.29 \\

All sinusoidal             & 8{,}640  & 1{,}080 & 2.70 & 7.20  & 15.33 & 20.17 & 1.45 & 3.87 \\

All frames                 & 12{,}480 & 1{,}560 & 2.14 & 8.35  & 20.32 & 25.86 & 0.12 & 2.45 \\
\bottomrule
\end{tabular}}
\end{table}

\begin{figure*}[pos=tp]
    \centering
    \includegraphics[width=0.9\linewidth]{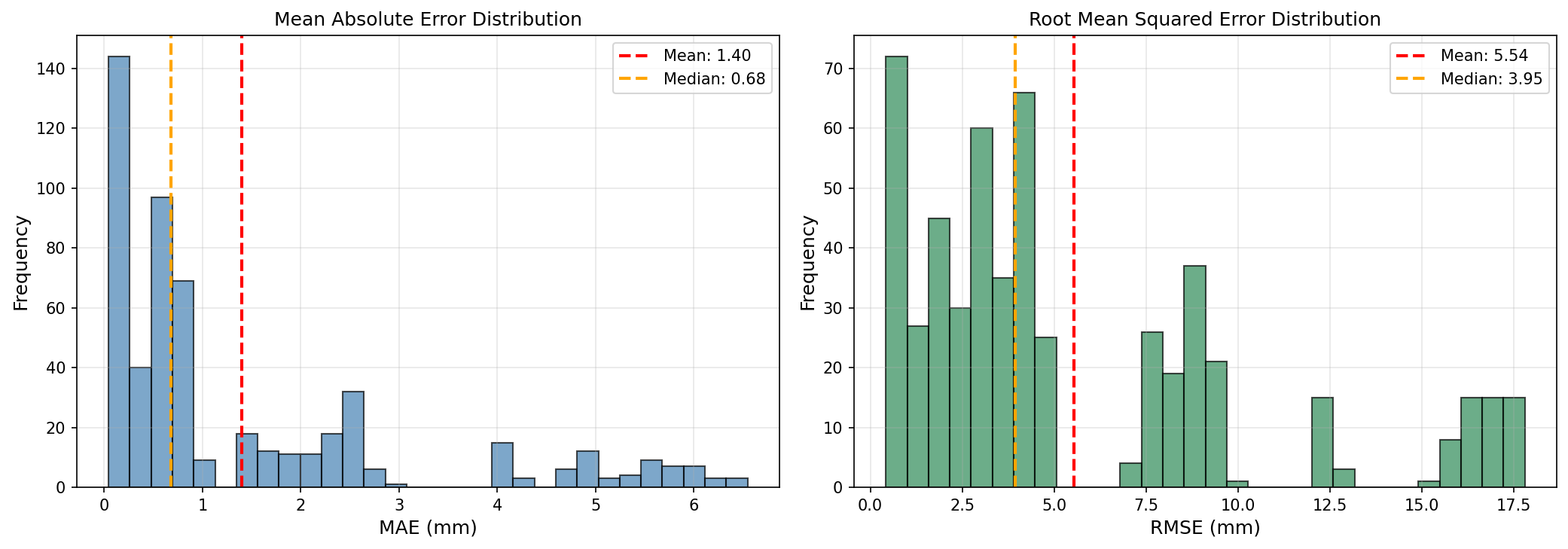}\\
    \vspace{0.3cm}
    \includegraphics[width=0.9\linewidth]{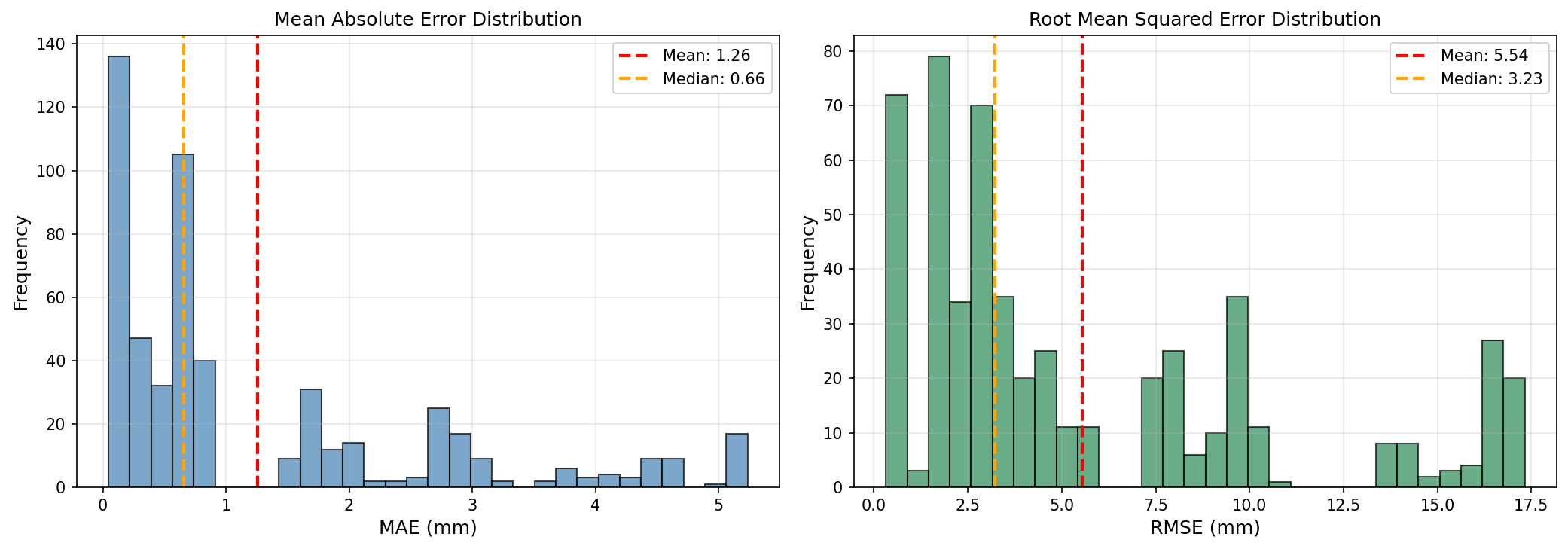}
    \caption{Overall error distributions for single-orientation
    multi-frame training conditions. Top: horizontal only (1.40~mm mean
    MAE). Bottom: vertical only (1.26~mm mean MAE). Both
    single-orientation conditions produce tight, low-error distributions
    comparable to each other.}
    \label{fig:multiframe_histograms_single}
\end{figure*}

\begin{figure*}[pos=tp]
    \centering
    \includegraphics[width=0.9\linewidth]{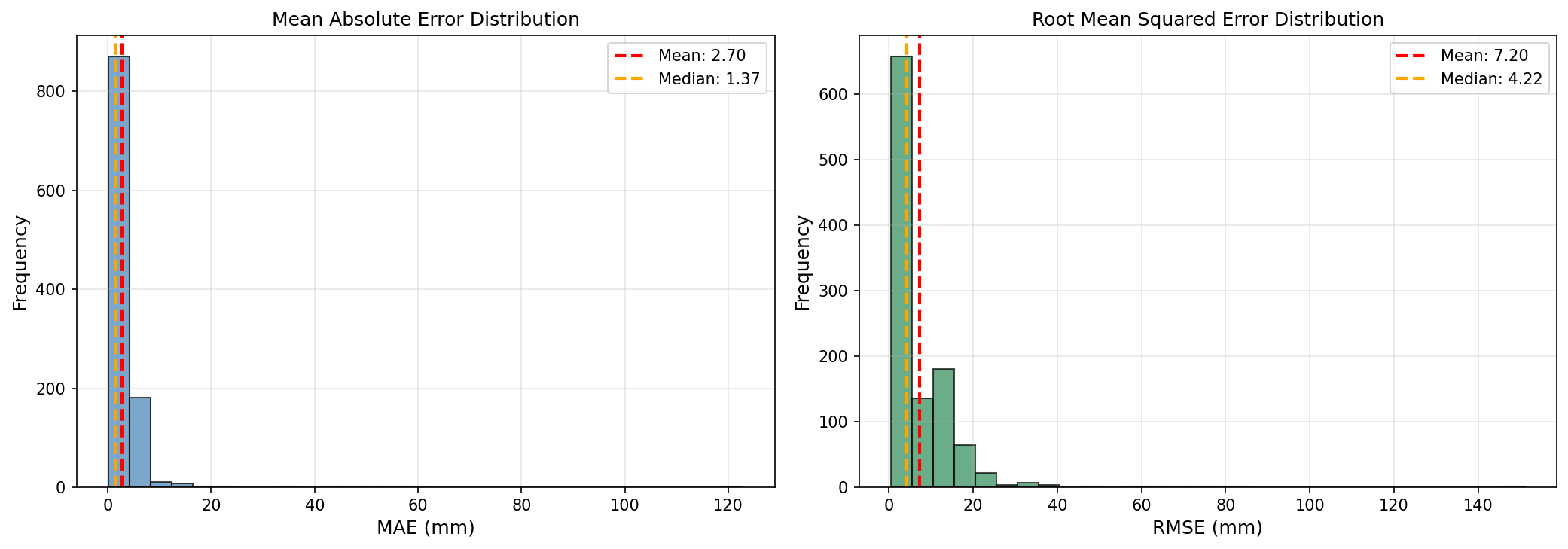}\\
    \vspace{0.3cm}
    \includegraphics[width=0.9\linewidth]{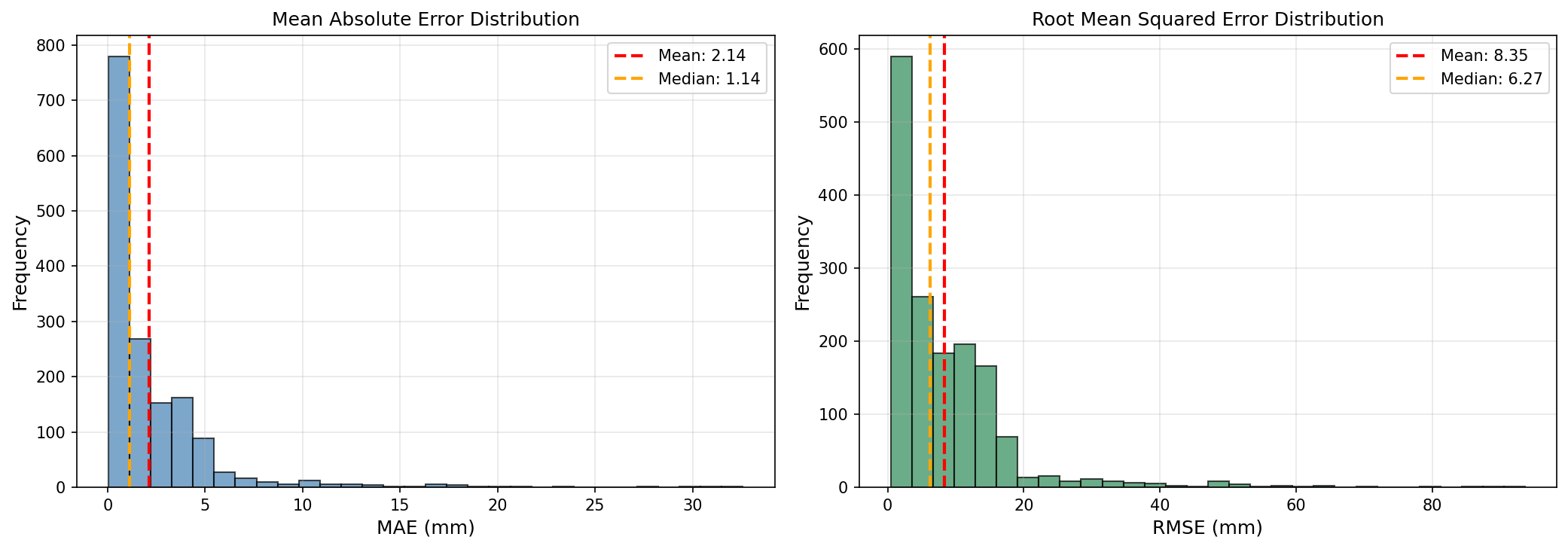}
    \caption{Overall error distributions for mixed multi-frame training
    conditions. Top: all sinusoidal (2.70~mm mean MAE). Bottom: all
    frames (2.14~mm mean MAE). Mixing fringe orientations or including
    non-sinusoidal patterns increases both error magnitude and variance
    compared to single-orientation training
    (Fig.~\ref{fig:multiframe_histograms_single}).}
    \label{fig:multiframe_histograms_mixed}
\end{figure*}

Training on all 18 phase-shifted frames from a single fringe orientation
yields the largest improvements. Horizontal-only training reduces object
MAE by 20.6\% (16.20~mm to 12.87~mm), while vertical-only training
achieves a 23.7\% reduction (16.20~mm to 12.36~mm). Both
single-orientation conditions also dramatically improve background
suppression, reducing background MAE by 77--83\% (0.92~mm to
0.16--0.21~mm). The improvement in background prediction is expected:
with 18~different phase patterns projected onto the background during
training, the network learns that varying fringe intensity on background
regions should always map to zero depth, regardless of phase. This
diversity forces the model to distinguish ``fringe pattern on object''
(which carries depth information) from ``fringe pattern on background''
(which should be suppressed). We note that the multi-frame conditions are
evaluated over larger test sets (540--1{,}560~samples) than the baseline
(30~samples), which reduces per-sample variance and yields more
statistically reliable estimates. The tighter error distributions
visible in Figs.~\ref{fig:multiframe_histograms_single}--\ref{fig:multiframe_histograms_mixed}
compared to the 30-sample baseline partly
reflect this difference in sample size.

Combining horizontal and vertical orientations (all sinusoidal, 8{,}640
samples) performs worse than either orientation alone despite having
$2\times$ more training data. Object MAE increases to 15.33~mm, only
marginally better than the 240-sample baseline. Including all 52~frames
further degrades object accuracy to 20.32~mm, worse than the baseline.
This suggests that the UNet struggles to learn a single mapping that
generalizes across fundamentally different fringe orientations.
Horizontal and vertical fringes encode depth information along
orthogonal axes, producing qualitatively different intensity patterns for
the same geometry. Rather than benefiting from this diversity, the
network appears to compromise between conflicting input-output
relationships, diluting the learned representation.

These results demonstrate that multi-frame training is beneficial only
when the additional frames are \emph{homogeneous} in orientation. Phase
diversity within a single orientation acts as effective data
augmentation, teaching the network phase-invariant depth extraction.
However, orientation diversity introduces conflicting gradients that
degrade performance. This finding has practical implications: when
expanding FPP training datasets, practitioners should prefer additional
phase-shifted frames from the same projection orientation rather than
mixing orientations or including structurally different patterns such as
gray code. For the remaining experiments in this paper, we continue using
the single-frame baseline dataset (240~training samples) to isolate the
effects of loss function and architecture choices independently of
dataset size. Although vertical-only training achieves the lowest
absolute error, the subsequent comparisons evaluate relative performance
between configurations rather than absolute accuracy, and the
conclusions regarding loss function ranking and architectural ordering
are expected to generalize across training data compositions.

\subsection{Loss Function Comparison}
\label{sec:loss_comparison}

Having established individual normalization as the optimal data
representation and confirmed that background fringes should be retained,
we now investigate whether alternative loss functions can further improve
reconstruction accuracy. We evaluate six L1/L2-based loss functions by
training the UNet architecture with individual normalized depth and full
fringe inputs:

\begin{enumerate}
    \item \textbf{RMSE loss (baseline):} Standard root mean squared
    error computed over all pixels:
    \begin{equation}
        \mathcal{L}_{\mathrm{RMSE}} = \sqrt{\frac{1}{HW}
        \sum_{u=1}^{W}\sum_{v=1}^{H}
        \bigl(\hat{D}(u,v)-D(u,v)\bigr)^2 + \epsilon},
    \end{equation}
    where $\epsilon=10^{-8}$ ensures numerical stability.

    \item \textbf{L1 loss:} Mean absolute error computed over all pixels,
    which is less sensitive to outliers than RMSE:
    \begin{equation}
        \mathcal{L}_{\mathrm{L1}} = \frac{1}{HW}
        \sum_{u=1}^{W}\sum_{v=1}^{H}
        \bigl|\hat{D}(u,v)-D(u,v)\bigr|.
    \end{equation}

    \item \textbf{Masked RMSE loss:} RMSE computed only on valid object
    pixels where $D(u,v)>0$. Let $\mathcal{M}^{+}\!=\!\{(u,v)\!:\!D(u,v)>0\}$:
    \begin{equation}
        \mathcal{L}_{\mathrm{MaskedRMSE}} =
        \sqrt{\frac{1}{|\mathcal{M}^{+}|}
        \!\!\sum_{(u,v)\in\mathcal{M}^{+}}\!\!(\hat{D}-D)^{\!2} + \epsilon}.
    \end{equation}

    \item \textbf{Masked L1 loss:} L1 loss computed only on valid
    pixels:
    \begin{equation}
        \mathcal{L}_{\mathrm{MaskedL1}} =
        \frac{1}{|\mathcal{M}^{+}|}
        \!\!\sum_{(u,v)\in\mathcal{M}^{+}}\!\!|\hat{D}-D|.
    \end{equation}

    \item \textbf{Hybrid RMSE loss:} Weighted combination of masked and
    global RMSE:
    \begin{equation}
        \mathcal{L}_{\mathrm{HybridRMSE}} =
        \alpha\cdot\mathcal{L}_{\mathrm{MaskedRMSE}}
        + (1-\alpha)\cdot\mathcal{L}_{\mathrm{RMSE}},
    \end{equation}
    with $\alpha$ as a tunable parameter to emphasize object regions
    while maintaining weak global regularization.

    \item \textbf{Hybrid L1 loss:} Weighted combination of masked and
    global L1:
    \begin{equation}
        \mathcal{L}_{\mathrm{HybridL1}} =
        \alpha\cdot\mathcal{L}_{\mathrm{MaskedL1}}
        + (1-\alpha)\cdot\mathcal{L}_{\mathrm{L1}}.
    \end{equation}
\end{enumerate}

The masked variants are designed to focus training explicitly on object
geometry by excluding background pixels from loss computation, while
hybrid losses balance object-focused learning with global regularization
to prevent pathological solutions such as extreme scale drift.

\subsubsection{Results}

Table~\ref{tab:loss_comparison} summarizes the results across six loss
function families. The results show
three distinct behaviors based on how losses balance object and
background pixels.

\begin{table}[pos=tbp]
\caption{Loss function comparison on individual normalized depth with
full fringe inputs (30 test samples). All errors in millimeters.}
\label{tab:loss_comparison}

\centering
\small
\resizebox{\ifdim\width>\columnwidth\columnwidth\else\width\fi}{!}{%
\begin{tabular}{@{}lcccccc@{}}
\toprule

\textbf{Loss Function} &
\textbf{Overall} & \textbf{Overall} &
\textbf{Object} & \textbf{Object} &
\textbf{Background} & \textbf{Background} \\

 & \textbf{MAE} & \textbf{RMSE} &
   \textbf{MAE} & \textbf{RMSE} &
   \textbf{MAE} & \textbf{RMSE} \\
\midrule
RMSE (Baseline)           &   2.30 &   6.80 & 16.20 & 21.19 & \textbf{0.92} & \textbf{3.00} \\

L1                         &   2.03 &   8.24 & 19.34 & 25.93 & 0.13 & 2.61 \\

Masked RMSE                & 122.38 & 673.08 & 18.83 & 23.02 & 135.67 & 714.82 \\

Masked L1                  & 147.28 & 524.76 & 22.65 & 27.97 & 163.53 & 557.26 \\

Hybrid RMSE ($\alpha$=0.5) &   2.91 &   7.36 & 15.29 & 19.40 & 1.64 & 4.80 \\

Hybrid RMSE ($\alpha$=0.7) &   3.94 &   8.39 & 14.55 & 18.69 & 2.86 & 6.49 \\

Hybrid RMSE ($\alpha$=0.9) &   6.00 &  10.92 & 15.05 & 19.19 & 5.24 & 9.83 \\

Hybrid L1 ($\alpha$=0.5)   &   2.90 &   8.99 & 15.41 & 19.37 & 1.53 & 7.09 \\

Hybrid L1 ($\alpha$=0.7)   & \textbf{3.31} & \textbf{9.85} & \textbf{14.54} & \textbf{17.88} & 2.01 & 8.44 \\

Hybrid L1 ($\alpha$=0.9)   &   4.66 &  12.77 & 14.73 & 18.85 & 3.63 & 11.93 \\
\bottomrule
\end{tabular}}
\end{table}

\begin{figure*}[pos=tp]
    \centering
    \includegraphics[width=0.95\linewidth]{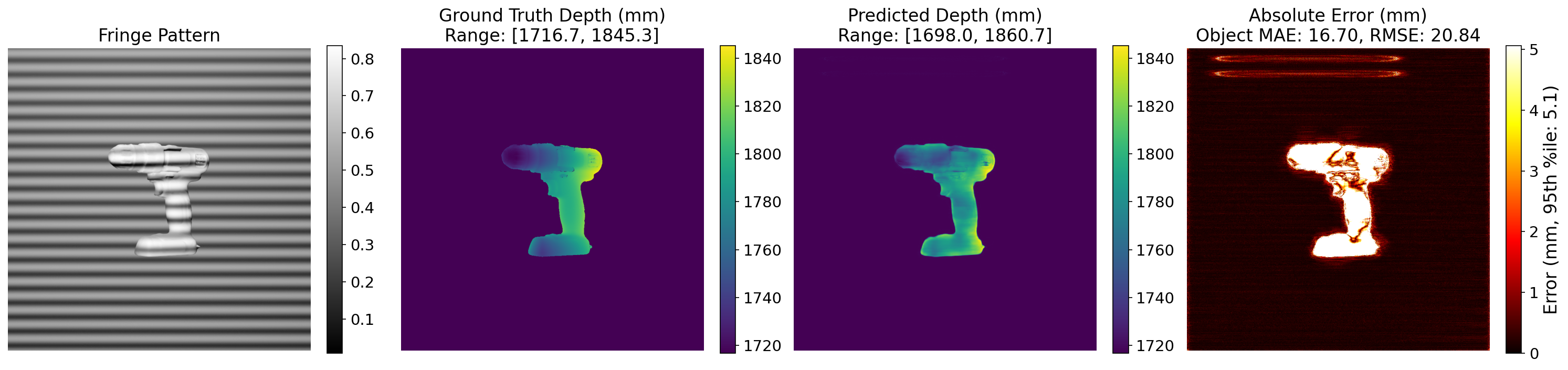}\\
    \vspace{0.2cm}
    \includegraphics[width=0.95\linewidth]{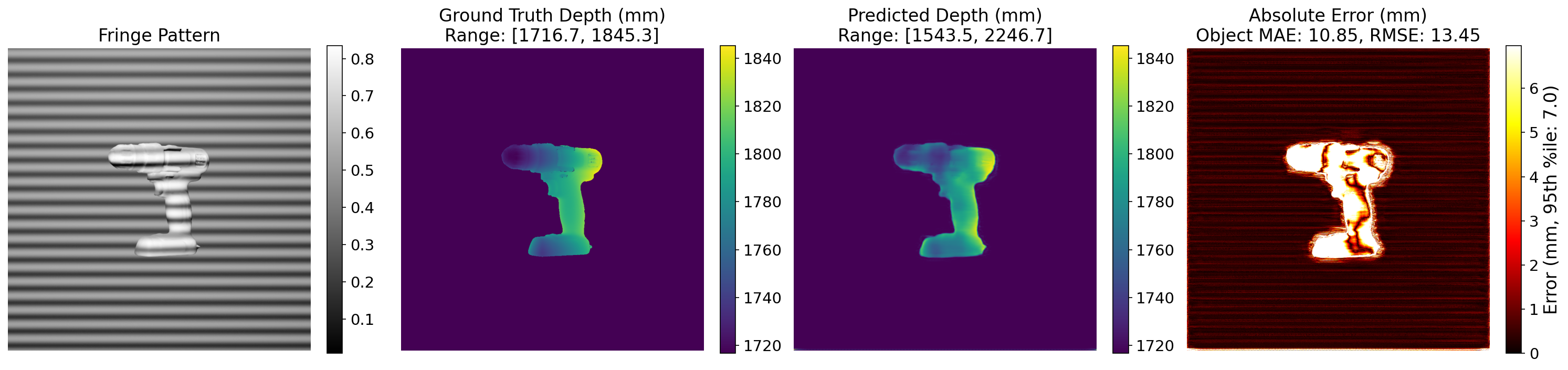}\\
    \vspace{0.2cm}
    \includegraphics[width=0.95\linewidth]{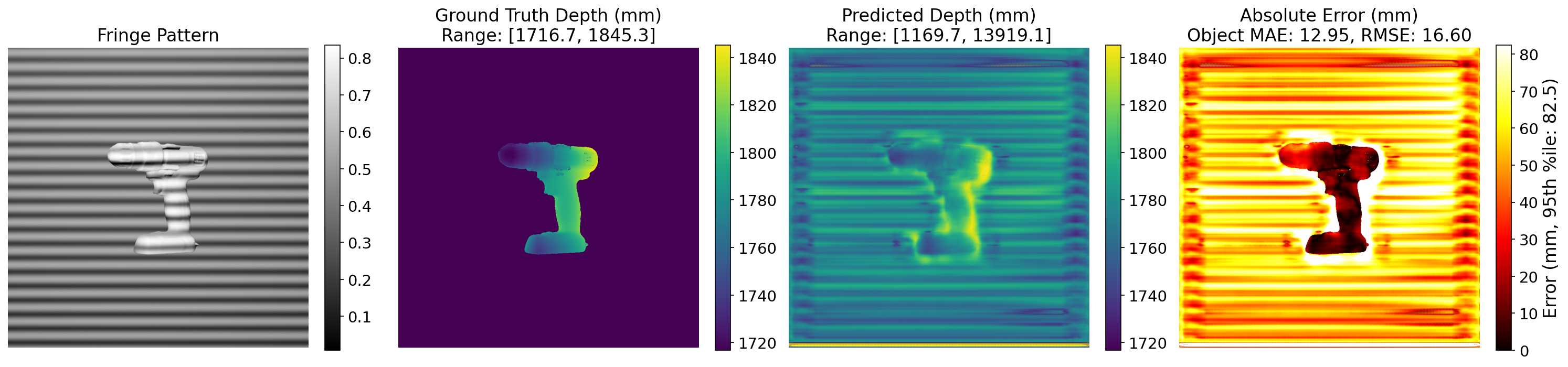}
    \caption{Single-shot depth reconstruction for the power drill object
    under three representative loss functions. Top: RMSE baseline
    (16.70~mm object MAE). Middle: Hybrid L1 with $\alpha=0.7$ (12.21~mm
    object MAE, 27\% improvement, best overall). Bottom: Masked RMSE
    (23.93~mm object MAE) shows scale drift artifacts in background
    despite lower nominal object error.}
    \label{fig:loss_predictions}
\end{figure*}

\noindent\textbf{Standard losses (RMSE, L1).} The baseline RMSE achieves
excellent overall error (2.30~mm MAE, 6.80~mm RMSE) with superior
background suppression (0.92~mm MAE) but moderate object error (16.20~mm
MAE). L1 loss improves background suppression further (0.13~mm MAE) but
increases object error to 19.34~mm (19\% worse), demonstrating the
robustness-accuracy tradeoff inherent to L1 versus L2 objectives.

\noindent\textbf{Masked losses (Masked RMSE, Masked L1).} These losses
catastrophically fail despite theoretically focusing on object geometry.
Masked RMSE produces 122.38~mm overall MAE with extreme background errors
(135.67~mm MAE, 714.82~mm RMSE), indicating severe scale drift.
Fig.~\ref{fig:loss_predictions} (bottom) shows horizontal banding
artifacts characteristic of unconstrained scale. This failure occurs
because masked losses provide no constraint on background predictions,
allowing the network to converge to arbitrary offsets that minimize
masked error at the expense of global consistency. The individual
normalization's per-sample scale compounds this issue.

This failure mode directly parallels the background fringe ablation study
(Sec.~\ref{sec:background_ablation}), where physically removing
background fringes degraded performance 2.8--7.3$\times$. Both failures
stem from the same mechanism: removing spatial context. Background fringe
removal creates artificial black-to-object transitions and eliminates the
continuous phase reference. Masked losses achieve a similar effect
through gradient masking. By providing zero gradients on background
pixels, the network learns to ignore background entirely, allowing
predictions to drift to arbitrary scales that minimize object-only error.
Both experiments demonstrate that background information, whether through
direct fringe patterns or regularization constraints, is essential for
stable depth prediction in single-shot FPP.

\noindent\textbf{Hybrid losses ($\alpha$ ablation).} Hybrid L1 with
$\alpha=0.7$ achieves the best object-only performance (14.54~mm MAE,
17.88~mm RMSE), representing a 10\% improvement over baseline RMSE
(16.20~mm MAE) and a 25\% improvement over L1 (19.34~mm MAE).
Fig.~\ref{fig:loss_predictions} (middle) shows clean predictions without
the scale drift artifacts of masked losses. Increasing $\alpha$ to 0.9
slightly degrades performance (14.73~mm object MAE), while reducing to
0.5 approaches baseline performance (15.41~mm MAE), demonstrating an
inverted-U relationship where moderate object emphasis is optimal. The
weak global regularization term $(1-\alpha)\cdot\mathcal{L}_{\mathrm{global}}$
prevents pathological solutions while the dominant masked term
$\alpha\cdot\mathcal{L}_{\mathrm{masked}}$ focuses training on object
geometry.

The modest improvements from hybrid losses (10--16\% object MAE
reduction) suggest that the baseline RMSE is already near-optimal for
the information available in single fringe images. The periodic nature of
fringe patterns and absence of temporal unwrapping fundamentally limit
achievable accuracy, and loss function engineering provides only marginal
gains. These results establish Hybrid L1 with $\alpha$=0.7 as the
best-performing configuration for subsequent architecture comparison.

\subsection{Architecture Comparison}
\label{sec:models}

Having identified individual normalization and Hybrid L1 loss
($\alpha$=0.7) as the optimal configuration, we now benchmark this setup
across four representative architectures to evaluate whether model design
can overcome the fundamental information limitations of single-shot
reconstruction. The optimizer and other hyperparameters were kept the
same as described in Sec.~\ref{sec:normalization_comparison}.

\noindent\textbf{UNet.}~\citep{Ronneberger2015} Encoder-decoder with skip
connections (described in Sec.~\ref{sec:normalization_comparison}).

\noindent\textbf{TransUNet.}~\citep{chen2021transunet} Identical encoder
and decoder to UNet, but with the convolutional bottleneck replaced by a
Vision Transformer (ViT). At the bottleneck ($60\times60$ spatial
resolution), the feature map is split into $6\times6$ patches (100
tokens), processed by 8 transformer layers with 12-head self-attention
(embedding dimension 768), then reshaped back to spatial dimensions.
This isolates the effect of global self-attention at the bottleneck
while keeping all other components identical to UNet (116.2M parameters
vs.\ 31.0M for UNet).

\noindent\textbf{ResUNet.}~\citep{ikeda2025deep} UNet with residual blocks
replacing convolutional blocks, four levels ($960\times960$ to
$120\times120$), and identity skip connections for improved gradient
flow.

\noindent\textbf{Pix2Pix.}~\citep{Isola2017} Conditional GAN with U-Net
generator and PatchGAN discriminator, adapted from NVIDIA
Pix2Pix-HD.~\citep{Wang2018}

\subsubsection{Results}

Table~\ref{tab:architecture_comparison} summarizes performance across
architectures. UNet achieves the best object-only error (14.54~mm MAE,
17.88~mm RMSE), outperforming TransUNet by 29\%, ResUNet by 60\%, and
Pix2Pix by 91\%. Fig.~\ref{fig:architecture_predictions} shows
representative predictions for a container bottle object.

\begin{table}[pos=tbp]
\caption{Architecture comparison on individual normalized depth with
Hybrid L1 loss ($\alpha$=0.7) across 30 test samples. All errors in
millimeters. UNet achieves best performance, but all models show
substantial object-only errors.}
\label{tab:architecture_comparison}

\centering
\small
\resizebox{\ifdim\width>\columnwidth\columnwidth\else\width\fi}{!}{%
\begin{tabular}{@{}lcccccc@{}}
\toprule

\textbf{Architecture} &
\textbf{Overall} & \textbf{Overall} &
\textbf{Object} & \textbf{Object} &
\textbf{Background} & \textbf{Background} \\

 & \textbf{MAE} & \textbf{RMSE} &
   \textbf{MAE} & \textbf{RMSE} &
   \textbf{MAE} & \textbf{RMSE} \\
\midrule
UNet    & \textbf{3.31} & \textbf{9.85} & \textbf{14.54} & \textbf{17.88} & \textbf{2.01} & \textbf{8.44} \\

TransUNet &  3.56 & 10.81 & 18.81 & 23.83 & 1.82 & 8.01 \\

ResUNet &  3.73 & 13.73 & 23.32 & 28.38 & 1.80 & 11.54 \\

Pix2Pix &  5.48 & 12.20 & 27.73 & 38.22 & 3.16 &  5.15 \\
\bottomrule
\end{tabular}}
\end{table}

\begin{figure*}[pos=tp]
    \centering
    \includegraphics[width=0.95\linewidth]{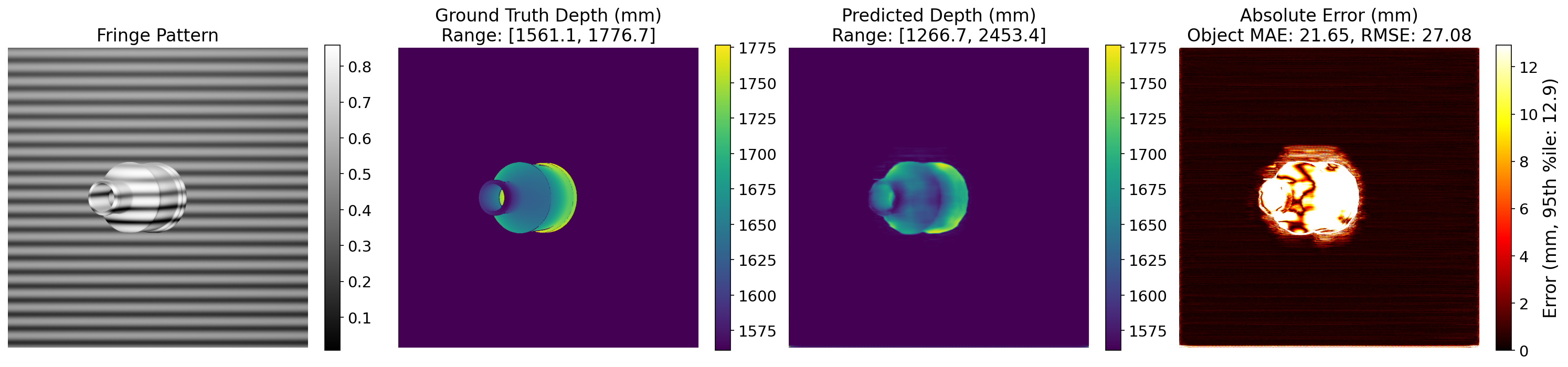}\\
    \vspace{0.2cm}
    \includegraphics[width=0.95\linewidth]{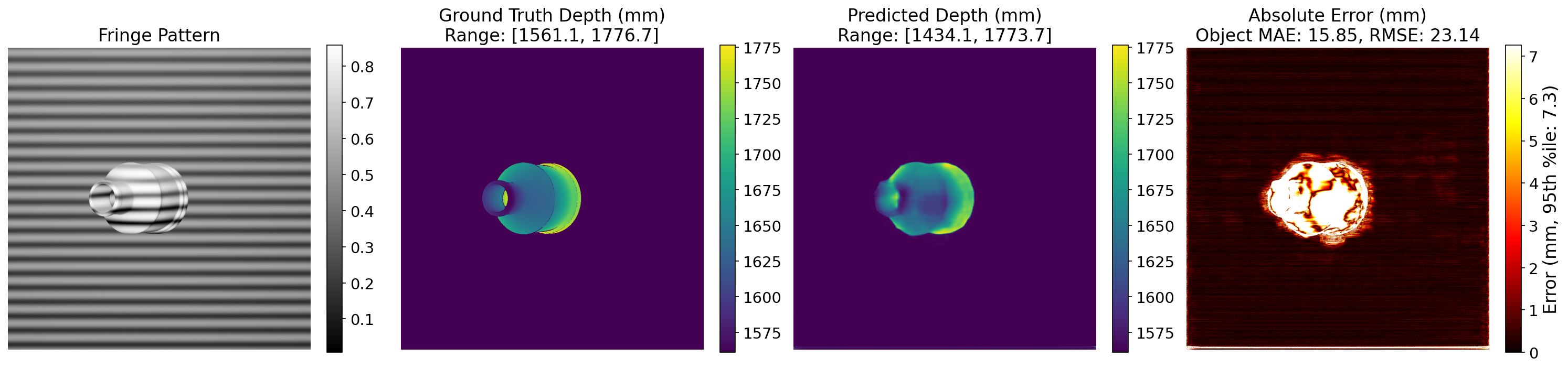}\\
    \vspace{0.2cm}
    \includegraphics[width=0.95\linewidth]{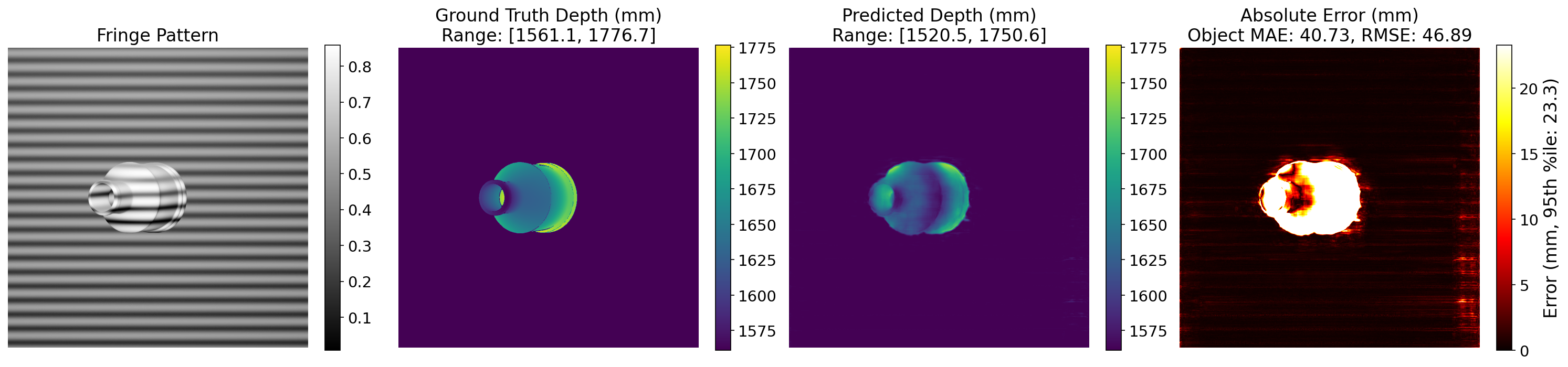}\\
    \vspace{0.2cm}
    \includegraphics[width=0.95\linewidth]{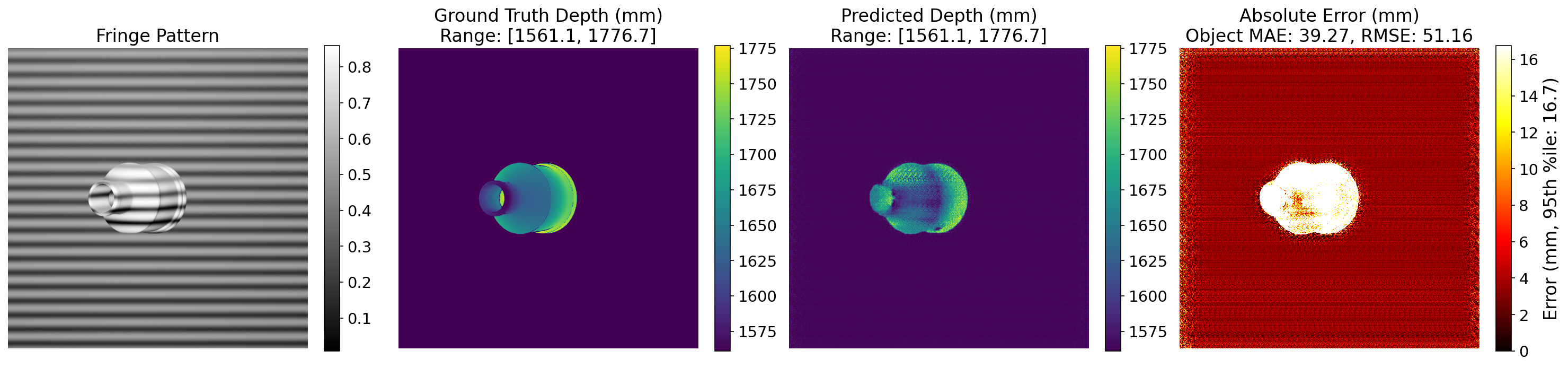}
    \caption{Single-shot depth reconstruction for the container bottle
    object across four architectures. From top to bottom: UNet (21.65~mm
    object MAE, 27.08~mm object RMSE), TransUNet (15.85~mm object MAE,
    23.14~mm object RMSE), ResUNet (40.73~mm object MAE, 46.89~mm object
    RMSE), Pix2Pix (39.27~mm object MAE, 51.16~mm object RMSE, worst
    performance). All models capture coarse geometry but fail on
    fine-scale accuracy. Note severe background artifacts in the Pix2Pix
    prediction, characteristic of adversarial training's focus on
    perceptual quality over metric accuracy.}
    \label{fig:architecture_predictions}
\end{figure*}

UNet's superior performance under the optimized configuration suggests
that the baseline encoder-decoder architecture with skip connections is
well-suited for this task when combined with proper normalization and
loss function design. The consistent $[0,1]$ normalized space and hybrid
loss allow UNet's simple architecture to focus learning on shape
reconstruction without the added complexity of scale variation or
sophisticated architectural components.

TransUNet provides a controlled test of whether global self-attention
improves depth prediction: since it shares the same encoder, decoder,
and skip connections as UNet, the only difference is the transformer
bottleneck. Despite having 3.7$\times$ more parameters (116.2M vs.\
31.0M), TransUNet achieves 29\% worse object MAE than UNet (18.81~mm
vs.\ 14.54~mm). This indicates that the convolutional bottleneck's
local inductive biases are better suited to this task than global
self-attention, particularly given the limited 240-sample training set.
ResUNet performs similarly to TransUNet (23.32~mm object MAE), suggesting
that neither residual connections nor transformer attention overcome the
fundamental information deficit of single-shot reconstruction. We note
that this controlled architectural comparison replaces the specific
method comparison reported in the conference version of this
benchmark~\citep{haroon2026fppml}, which contrasted UNet against Hformer;
substituting a controlled transformer-bottleneck variant (TransUNet) for
a distinct published method isolates the effect of the architectural
change itself and accordingly yields a different relative gap (29\% here
versus the 52\% reported there against Hformer).

Pix2Pix performs worst (27.73~mm object MAE, 38.22~mm object RMSE),
revealing a fundamental misalignment between adversarial training and
metrological objectives. Fig.~\ref{fig:architecture_predictions} shows a
counterintuitive result: Pix2Pix produces visually compelling predictions
with smooth surfaces and statistically correct depth distributions. The
container bottle prediction exhibits a depth range [1561.1, 1776.7]~mm
that exactly matches ground truth, yet achieves nearly 2$\times$ worse
error than UNet (14.54~mm MAE). This occurs because the adversarial
discriminator optimizes for perceptual plausibility (``does this
\emph{look} like a depth map?'') rather than point-wise accuracy. The
network learns to match statistical properties (min, max, mean depth) and
produce realistic surface geometries, but introduces smooth systematic
offsets of 20--40~mm that are visually imperceptible yet catastrophic for
precision measurement. These offsets appear as uniform color shifts in
visualizations rather than high-frequency noise, making predictions
appear deceptively accurate to human observers while failing metrological
evaluation. The severe background artifacts further confirm this: the
network learned that depth maps should have varying depth values, without
the critical constraint that background regions must predict exactly zero
depth.

Critically, even the best architecture (UNet at 14.54~mm MAE) remains
far from the sub-millimeter accuracy of traditional FPP. For objects with
an 80~mm depth range in FPP-ML-Bench, this represents 18\% relative error.
The modest 1.9$\times$ performance gap between best (UNet: 14.54~mm) and
worst (Pix2Pix: 27.73~mm) architectures, combined with all models'
failure to achieve precision depth reconstruction, confirms our
hypothesis: the fundamental limitation is information deficit, not model
design. Single fringe images lack sufficient information for accurate
depth recovery, and no amount of architectural engineering can overcome
this constraint. This is precisely the regime Proposition~\ref{prop:noninjective} of Section~\ref{sec:singleshot_theory} predicts is irreducibly ill-posed: every architecture in the comparison is approximating a mapping that the underlying physics does not uniquely determine.

This finding parallels the loss function results
(Sec.~\ref{sec:loss_comparison}), where even optimal loss design provided
only 10\% improvement over baseline. Together, these results demonstrate
that single-shot fringe-to-depth mapping without explicit phase
information is fundamentally limited. Networks learn coarse shape priors
and statistical regularities rather than accurate geometry, regardless of
architectural sophistication.

\section{Mechanistic Interpretability Analysis}
\label{sec:interpretability}

The previous sections established that even the best-performing model (UNet with Hybrid L1 loss) achieves only 14.54~mm object MAE, representing 18\% of the object depth range. To understand \emph{why} this performance gap exists, we employ mechanistic interpretability, a set of techniques from machine learning research that aim to understand \emph{how} neural networks arrive at their predictions by examining their internal representations and computational mechanisms.~\citep{olah2020zoom} To our knowledge, this is the first application of mechanistic interpretability to FPP.

Unlike standard evaluation metrics that only assess prediction accuracy, mechanistic interpretability provides insight into the features and strategies a network has learned. This distinction is critical: a network may achieve reasonable accuracy through shortcuts or heuristics that differ fundamentally from the intended solution. In our context, a network trained to predict depth from fringe patterns could potentially learn two very different strategies: (1)~\textbf{physics-based decoding}, where the network extracts phase information from fringe intensity variations and converts it to depth following the mathematical relationship governing FPP; or (2)~\textbf{geometric shortcut learning}, where the network recognizes object shapes from the training data and predicts depth by filling in learned shape templates, bypassing the fringe pattern entirely.

We hypothesize that our networks learn the geometric shortcut rather than the physics-based solution. To test this, we employ three complementary interpretability techniques: linear probing to examine what information is encoded in intermediate representations, gradient-weighted class activation mapping (GradCAM) to visualize spatial attention patterns, and out-of-distribution testing with a featureless flat plane to expose reliance on geometric cues.

\subsection{Linear Probing Analysis}

Linear probing is a widely-used interpretability technique that determines what information is accessible in a network's internal representations.~\citep{alain2016understanding} The approach trains small auxiliary ``probe'' networks to predict specific target quantities from frozen intermediate activations. If a probe can accurately predict a target, that information must be explicitly encoded (or easily computable from) the activations; if the probe fails, the information is either absent or encoded in a form that is difficult to extract.

Crucially, linear probing reveals what information the network \emph{represents}, not merely what it receives as input. A network processing fringe images necessarily receives intensity patterns as input, but this does not mean those patterns are preserved or utilized in deeper layers. By probing intermediate activations, we can determine whether the network maintains fringe pattern information for depth computation or discards it in favor of geometric features.

We train spatial probes to predict two types of targets from UNet activations: (1)~\textbf{geometric features} (edge maps computed via Sobel filtering of the ground truth depth), representing structural boundary information; and (2)~\textbf{depth values}, representing the physical quantity the network should predict.

For each of the nine UNet layers (four encoder, bottleneck, four decoder), we train a probe network consisting of four convolutional layers to predict either edges or depth from frozen activations. Probe performance is measured by validation MSE loss, where lower loss indicates the target is more easily decoded from the activations.

Table~\ref{tab:probing_results} and Fig.~\ref{fig:probing_analysis} show that geometric features (edges) are consistently easier to decode than depth values across all layers. The average probe loss for edges (0.000379) is 2.82$\times$ lower than for depth (0.001069), indicating that the network's internal representations encode geometric structure more explicitly than depth values. This asymmetry is largest in the encoder and narrows toward parity in the decoder (dec2, dec3), where depth is most decodable and where the network produces its output, indicating that the network consolidates a shape-based depth representation in the decoder rather than refining depth from fringe phase.

\begin{table}[pos=tbp]
\caption{Linear probing validation MSE loss for geometry (edges) versus depth prediction across UNet layers. Lower loss indicates information is more easily decoded. Edges are consistently easier to predict than depth values.}
\label{tab:probing_results}

\centering
\small
\resizebox{\ifdim\width>\columnwidth\columnwidth\else\width\fi}{!}{%
\begin{tabular}{@{}lccc@{}}
\toprule
\textbf{Layer} & \textbf{Edges Loss} & \textbf{Depth Loss} & \textbf{Ratio (Depth/Edges)} \\
\midrule
enc1 (skip1)    & 0.000387 & 0.002284 & 5.90 \\

enc2 (skip2)    & 0.000343 & 0.001765 & 5.15 \\

enc3 (skip3)    & 0.000336 & 0.001391 & 4.14 \\

enc4 (skip4)    & 0.000399 & 0.000888 & 2.23 \\

bottleneck      & 0.000468 & 0.001099 & 2.35 \\

dec1 (up1)      & 0.000439 & 0.000882 & 2.01 \\

dec2 (up2)      & 0.000375 & 0.000411 & 1.10 \\

dec3 (up3)      & 0.000308 & 0.000371 & 1.20 \\

dec4 (up4)      & 0.000355 & 0.000527 & 1.48 \\

\textbf{Average} & \textbf{0.000379} & \textbf{0.001069} & \textbf{2.82} \\
\bottomrule
\end{tabular}}
\end{table}

\begin{figure}[pos=tbp]
    \centering
    \includegraphics[width=0.95\linewidth]{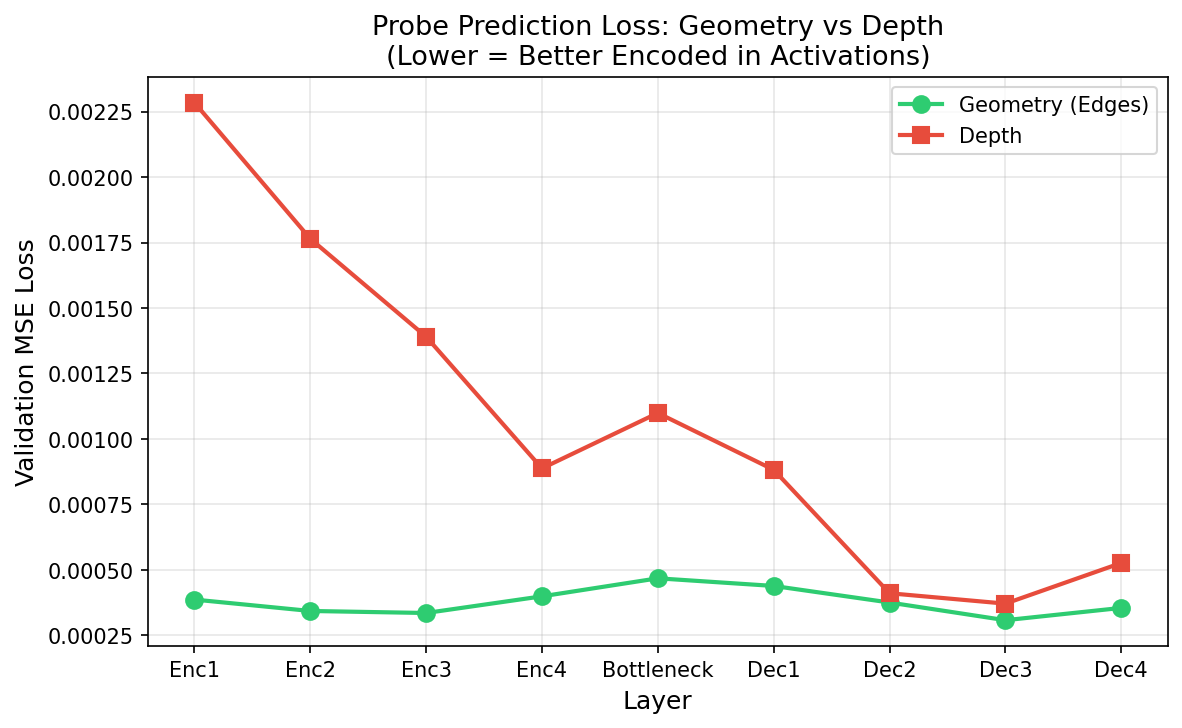}
    \caption{Linear probing analysis comparing how easily geometric (edges) versus physical (depth) information can be decoded from UNet activations at each layer. Edges (green) are consistently easier to predict than depth values (red), with an average ratio of 2.8$\times$ across all layers.}
    \label{fig:probing_analysis}
\end{figure}

\subsection{GradCAM Attention Analysis}

While linear probing reveals what information is \emph{encoded} in network activations, it does not directly show what information the network \emph{uses} for prediction. To address this, we employ Gradient-weighted Class Activation Mapping (GradCAM),~\citep{selvaraju2017grad} a technique that visualizes which spatial regions of the input most strongly influence the network's output.

GradCAM works by computing gradients of the output with respect to intermediate feature maps, then weighting each feature channel by its average gradient magnitude. The resulting heatmap highlights regions where small changes would most affect the prediction, effectively showing where the network ``pays attention'' when making decisions. Originally developed for classification tasks, GradCAM can be adapted for regression by computing gradients with respect to the mean predicted value.

For our analysis, we compute GradCAM heatmaps at multiple layers throughout the UNet and measure their correlation with two reference maps: (1)~\textbf{edge maps} derived from ground truth depth via Sobel filtering, representing geometric boundaries; and (2)~\textbf{fringe intensity variation maps} computed as the local standard deviation of fringe intensity, capturing where fringe patterns carry the most information.

The logic of this comparison is straightforward: if the network learned physics-based fringe decoding, its attention should correlate with fringe patterns, the input signal that encodes depth information. Conversely, if the network relies on geometric shortcuts, its attention should correlate with edges, the boundaries that define object shapes but do not directly encode depth in the physics of FPP.

\begin{table}[pos=tbp]
\caption{GradCAM correlation analysis comparing attention to edges versus fringe patterns across UNet layers (30 test samples). Ratio $>1$ indicates stronger attention to geometric boundaries than fringe patterns.}
\label{tab:gradcam_results}

\centering
\small
\resizebox{\ifdim\width>\columnwidth\columnwidth\else\width\fi}{!}{%
\begin{tabular}{@{}lccc@{}}
\toprule
\textbf{Layer} & \textbf{r(CAM, Edges)} & \textbf{r(CAM, Fringes)} & \textbf{Ratio} \\
\midrule
enc3    & 0.011 & $-$0.010 & 1.13 \\

enc4    & 0.242 &    0.183 & 1.33 \\

bottleneck & 0.144 & 0.130 & 1.10 \\

dec1    & 0.068 &    0.071 & 0.96 \\

dec3    & $-$0.007 & 0.007 & $-$0.72 \\

dec4    & 0.361 &    0.261 & 1.38 \\

\textbf{Average} & \textbf{0.136} & \textbf{0.107} & \textbf{1.28} \\
\bottomrule
\end{tabular}}
\end{table}

\begin{figure*}[pos=tp]
    \centering
    \includegraphics[width=0.95\linewidth]{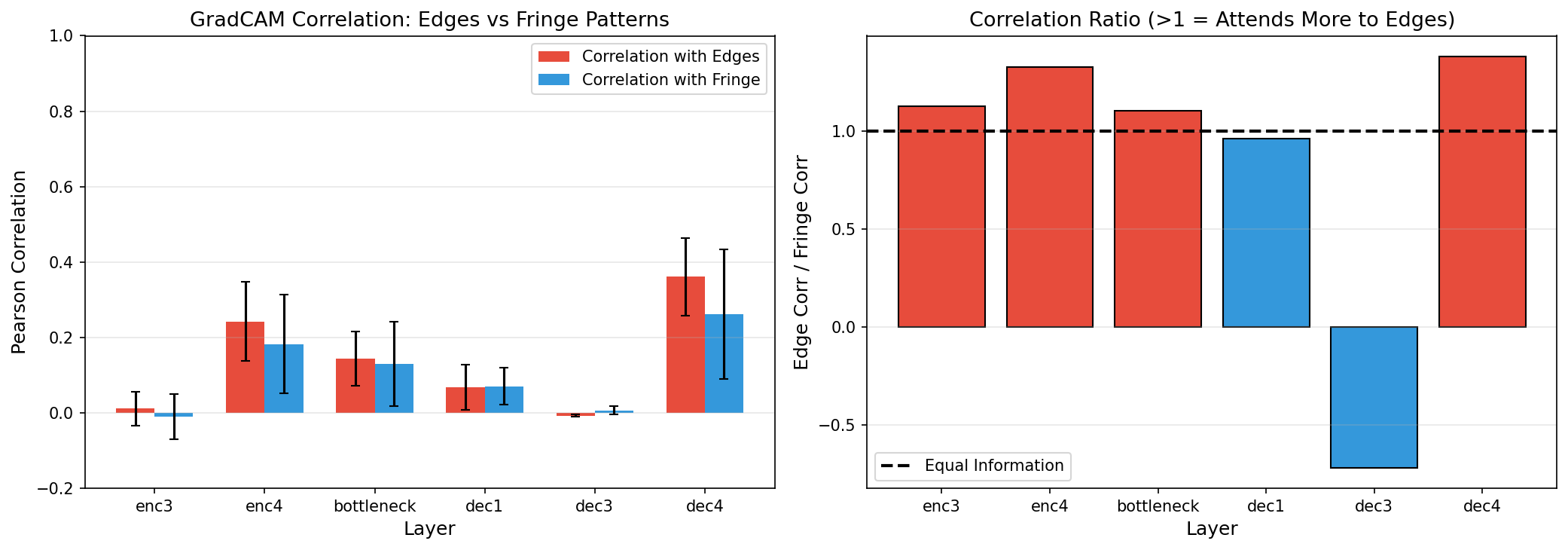}
    \caption{GradCAM attention analysis showing correlation of GradCAM heatmaps with edge maps versus fringe patterns across layers (30 test samples). The network attends more strongly to geometric boundaries than fringe intensity patterns, with an average edge/fringe correlation ratio of 1.28.}
    \label{fig:gradcam_analysis}
\end{figure*}

\begin{figure*}[pos=tp]
    \centering
    \begin{subfigure}[t]{0.13\linewidth}
        \centering
        \includegraphics[width=\linewidth]{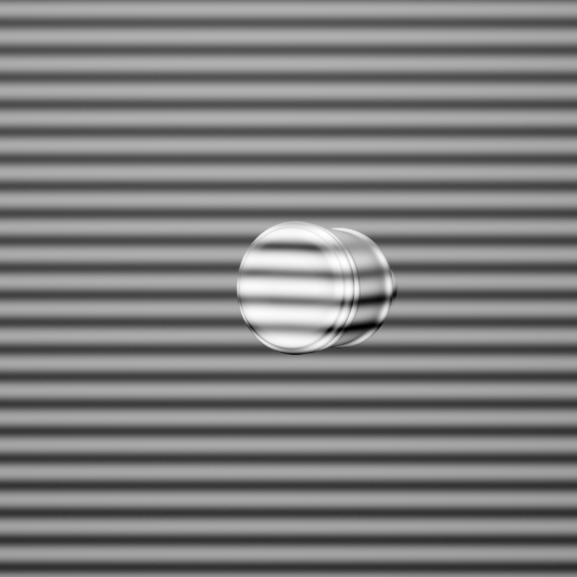}
        \caption{Input}
    \end{subfigure}\hfill
    \begin{subfigure}[t]{0.13\linewidth}
        \centering
        \includegraphics[width=\linewidth]{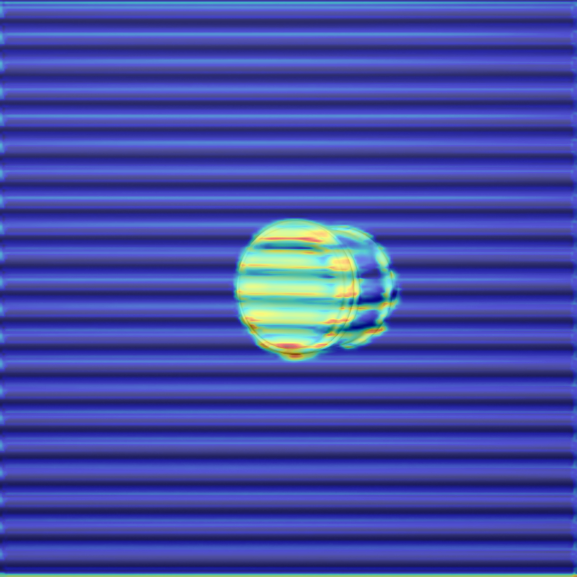}
        \caption{enc3}
    \end{subfigure}\hfill
    \begin{subfigure}[t]{0.13\linewidth}
        \centering
        \includegraphics[width=\linewidth]{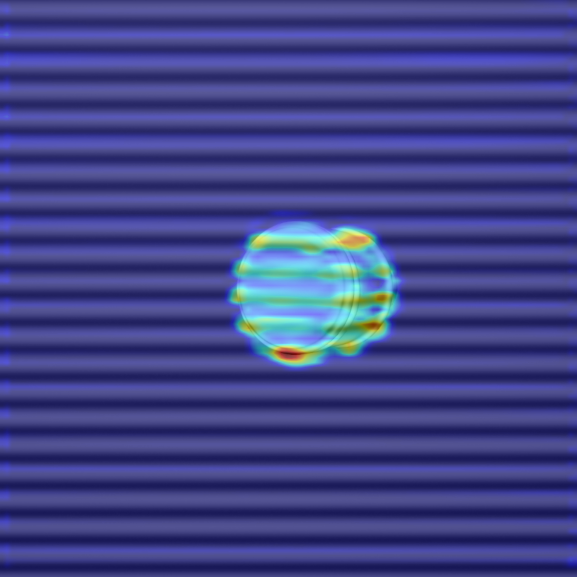}
        \caption{enc4}
    \end{subfigure}\hfill
    \begin{subfigure}[t]{0.13\linewidth}
        \centering
        \includegraphics[width=\linewidth]{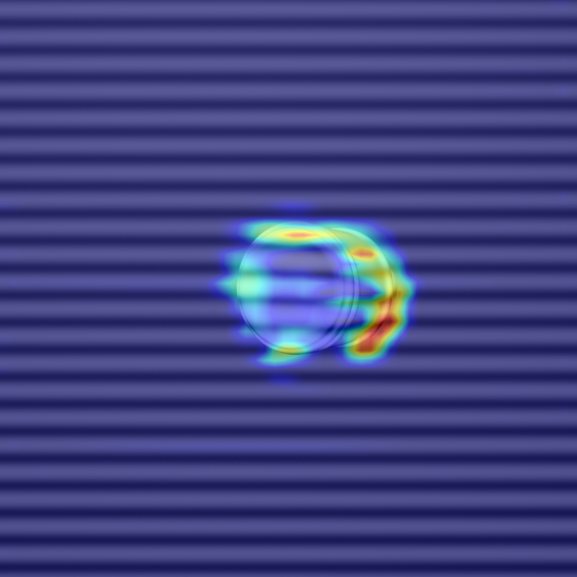}
        \caption{bottleneck}
    \end{subfigure}\hfill
    \begin{subfigure}[t]{0.13\linewidth}
        \centering
        \includegraphics[width=\linewidth]{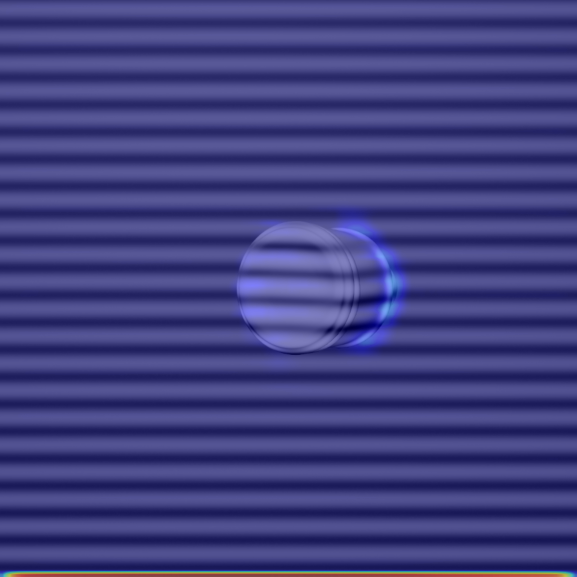}
        \caption{dec1}
    \end{subfigure}\hfill
    \begin{subfigure}[t]{0.13\linewidth}
        \centering
        \includegraphics[width=\linewidth]{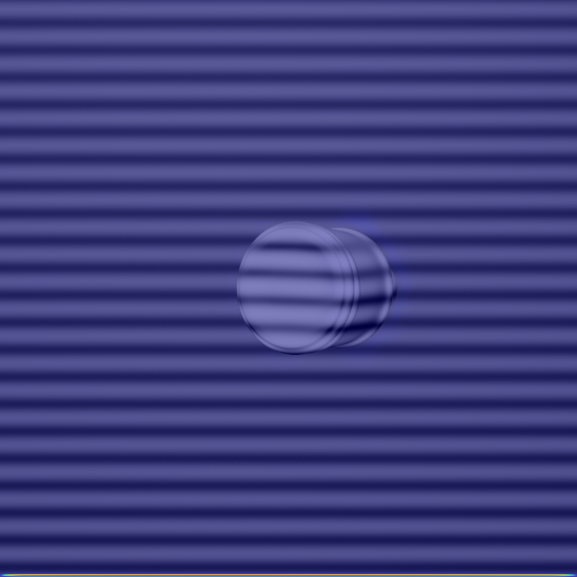}
        \caption{dec3}
    \end{subfigure}\hfill
    \begin{subfigure}[t]{0.13\linewidth}
        \centering
        \includegraphics[width=\linewidth]{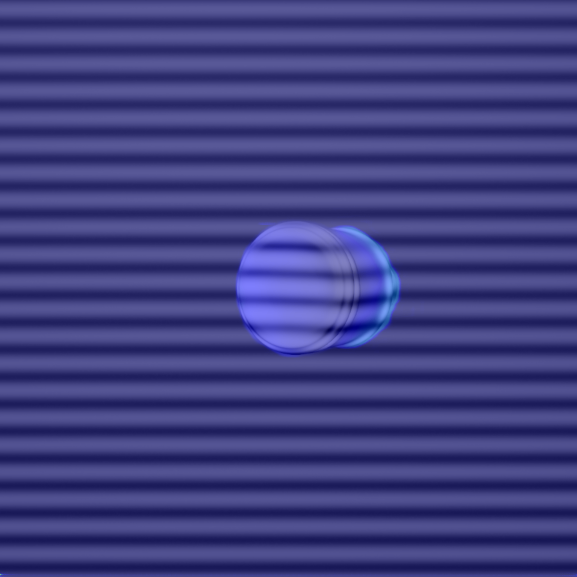}
        \caption{dec4}
    \end{subfigure}
    \caption{Example GradCAM visualization showing network attention across layers for a test sample. The network progressively focuses on object boundaries, with strongest attention at edges in enc4 and dec4 layers.}
    \label{fig:gradcam_sample}
\end{figure*}

Table~\ref{tab:gradcam_results} and Fig.~\ref{fig:gradcam_analysis} show that across 30 test samples, GradCAM correlates more strongly with edges than fringe patterns in five of six layers. However, the absolute correlation values are modest (average r = 0.136 for edges, r = 0.107 for fringes), reflecting the inherent difficulty of correlating a high-dimensional heatmap against a single scalar reference map. The direction of the effect is consistent with the geometric shortcut hypothesis, but the magnitude alone is not conclusive. The result is better interpreted alongside the linear probing and flat plane evidence: the GradCAM data show a systematic \emph{directional} preference for edges over fringes that, combined with the probing asymmetry (2.82$\times$) and the OOD failure, builds a coherent picture. Notably, the attention preference is strongest in the final decoder layer (dec4, ratio 1.38), where the network makes its final spatial prediction, and weakest or even reversed at dec3 (ratio $-$0.72), suggesting the network does not attend uniformly to edges throughout. Fig.~\ref{fig:gradcam_sample} shows an example where attention progressively concentrates on object boundaries through the encoder stages.

\subsection{Flat Plane Out-of-Distribution Test}

The most direct way to test whether a network has learned a shortcut versus the intended solution is to present inputs where the shortcut fails but the correct solution should still work. This out-of-distribution (OOD) testing paradigm~\citep{geirhos2020shortcut} is a cornerstone of mechanistic interpretability: by carefully constructing test cases that dissociate the shortcut from the correct answer, we can definitively determine which strategy the network employs.

For our hypothesis, the critical test case is an image containing valid fringe patterns but \emph{no geometric features}, specifically a flat, featureless plane. If the network learned physics-based fringe decoding, it should correctly predict uniform depth corresponding to the plane's distance; the fringe pattern encodes this depth regardless of whether edges are present. However, if the network relies on geometric shortcuts, it should fail catastrophically: with no edges to detect, the shape-based strategy has no basis for prediction.

We captured a flat plane at 1.8~m, well within the trained depth range (Fig.~\ref{fig:flat_plane_test}). A physics-aware network should predict \emph{uniform} depth across the plane corresponding to its distance from the projector. The key diagnostic is whether the network recognizes that a flat plane should produce flat depth.

The results are striking. Instead of predicting a uniform depth that reflects the plane's actual 1.8~m distance, the network predicts near-zero depth across most of the surface (mean = 0.088 normalized units), effectively classifying the flat plane itself as background. The prediction range extends to [$-$1.19, 4.38], far outside the valid [0,~1] training range, with substantial portions of the image predicting physically impossible negative depth values. This behavior is consistent with a network that has learned to associate ``no internal shape features'' with ``no object,'' defaulting to background depth whenever the geometric shortcut has nothing to latch onto.

GradCAM analysis (Fig.~\ref{fig:gradcam_flat_plane}) provides additional context. Encoder layers (enc3, enc4) show strong activation across the fringes \emph{within} the rectangular front plane, with comparatively little response to the fringes on the surrounding background. In contrast, the decoder layers (dec1, dec3, dec4) show minimal spatial attention. The network uses the rectangular boundary to define a region of interest and processes the fringe pattern selectively within it, yet still fails to convert that fringe information into a uniform depth estimate consistent with a flat surface. This is consistent with how the network processes images of complex objects (Table~\ref{tab:gradcam_results}), where edge-correlated attention defines an object region and shape-based depth filling proceeds within it. With no learned internal shape template to fill in for a featureless plane, the network's pipeline collapses to a near-zero (background) prediction even though encoder-stage attention is locally engaged with the fringe pattern.

\begin{figure}[pos=tbp]
    \centering
    \begin{subfigure}[c]{0.42\linewidth}
        \centering
        \includegraphics[width=\linewidth]{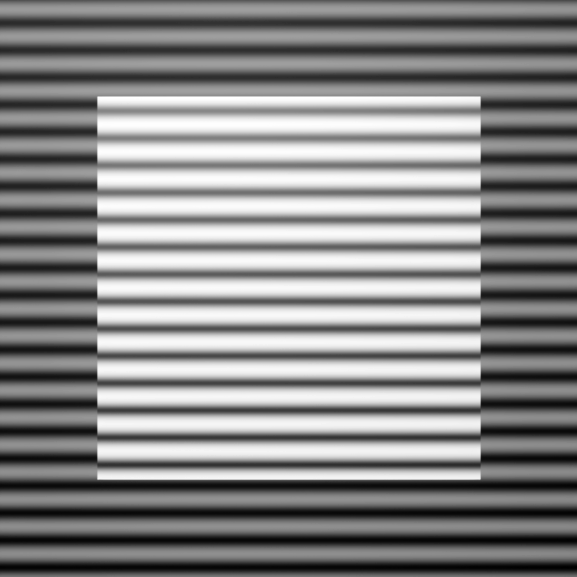}
        \caption{Input fringe image}
    \end{subfigure}\hfill
    \begin{subfigure}[c]{0.48\linewidth}
        \centering
        \includegraphics[width=\linewidth]{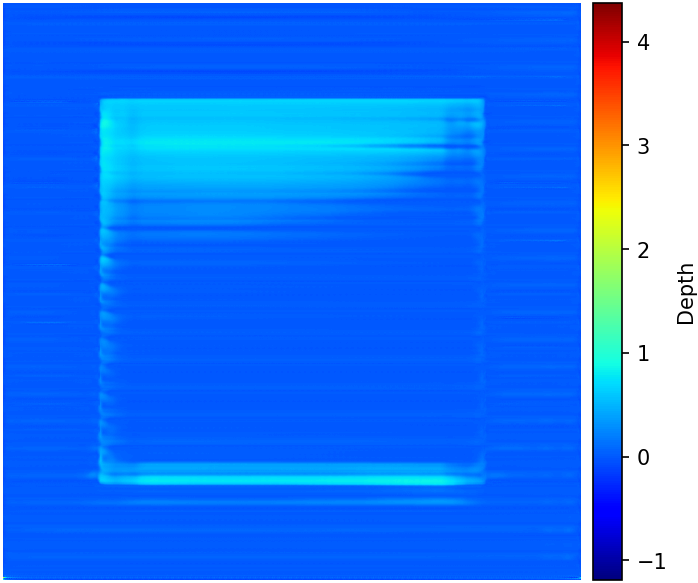}
        \caption{Network prediction (range: [-1.19, 4.38])}
    \end{subfigure}
    \caption{Flat plane out-of-distribution test. Left: Input fringe image of a flat plane captured at 1.8~m, within the trained depth range. Right: Network prediction showing a non-uniform depth map clustered near zero (mean = 0.088) with an extreme range [$-$1.19, 4.38] far outside the valid [0,~1] training range, indicating the network fails to recognize the flat surface and defaults to background-level predictions across most of it.}
    \label{fig:flat_plane_test}
\end{figure}

\begin{figure*}[pos=tp]
    \centering
    \begin{subfigure}[t]{0.13\linewidth}
        \centering
        \includegraphics[width=\linewidth]{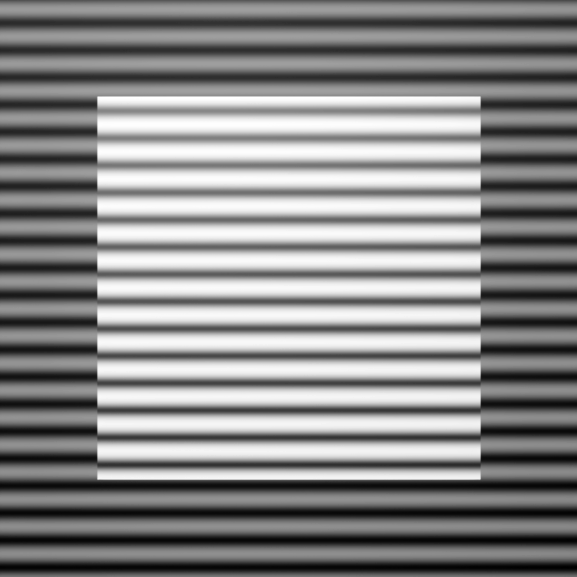}
        \caption{Input}
    \end{subfigure}\hfill
    \begin{subfigure}[t]{0.13\linewidth}
        \centering
        \includegraphics[width=\linewidth]{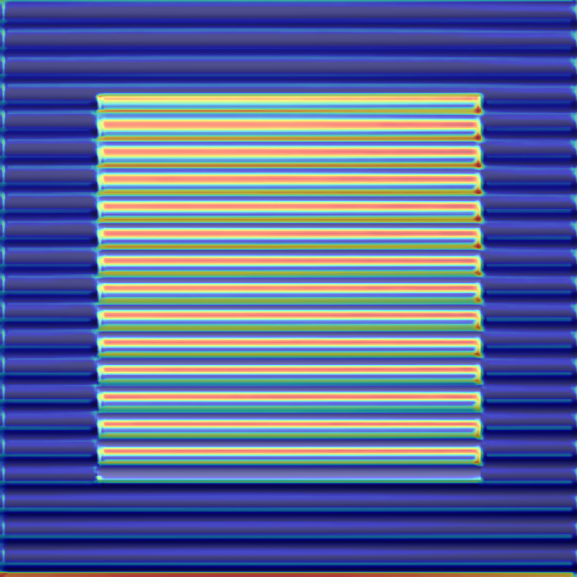}
        \caption{enc3}
    \end{subfigure}\hfill
    \begin{subfigure}[t]{0.13\linewidth}
        \centering
        \includegraphics[width=\linewidth]{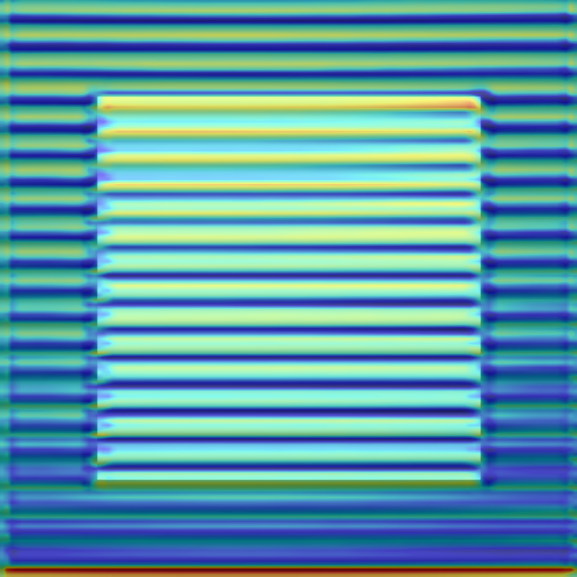}
        \caption{enc4}
    \end{subfigure}\hfill
    \begin{subfigure}[t]{0.13\linewidth}
        \centering
        \includegraphics[width=\linewidth]{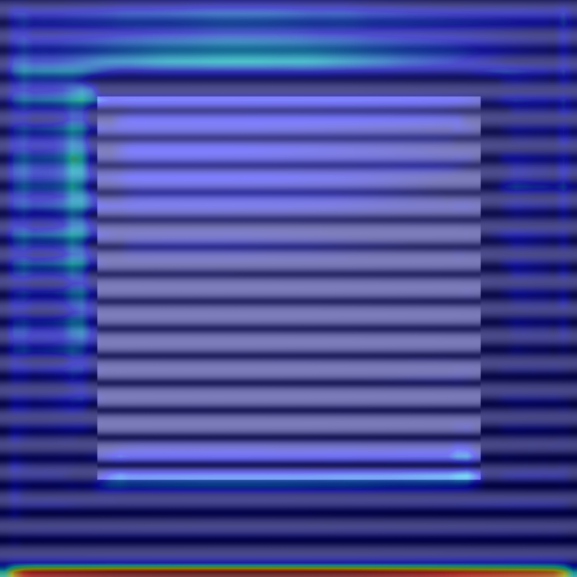}
        \caption{bottleneck}
    \end{subfigure}\hfill
    \begin{subfigure}[t]{0.13\linewidth}
        \centering
        \includegraphics[width=\linewidth]{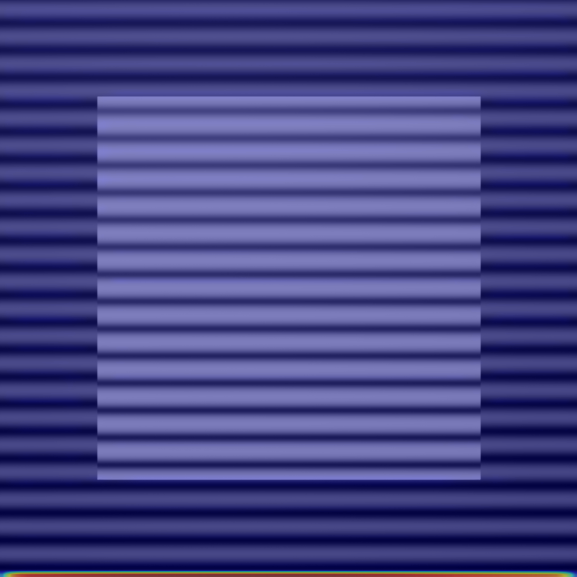}
        \caption{dec1}
    \end{subfigure}\hfill
    \begin{subfigure}[t]{0.13\linewidth}
        \centering
        \includegraphics[width=\linewidth]{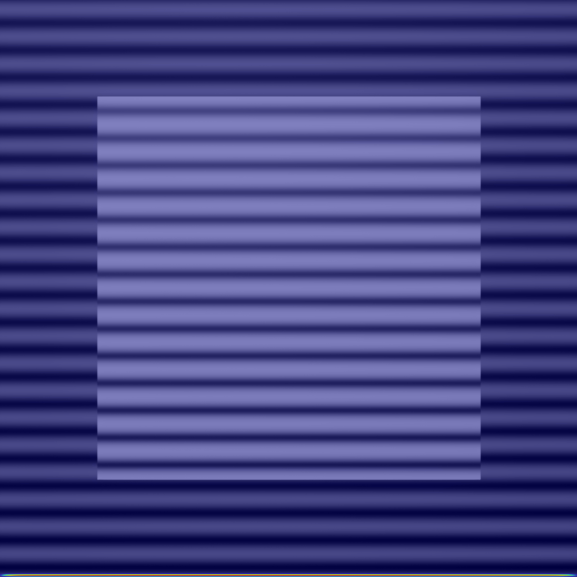}
        \caption{dec3}
    \end{subfigure}\hfill
    \begin{subfigure}[t]{0.13\linewidth}
        \centering
        \includegraphics[width=\linewidth]{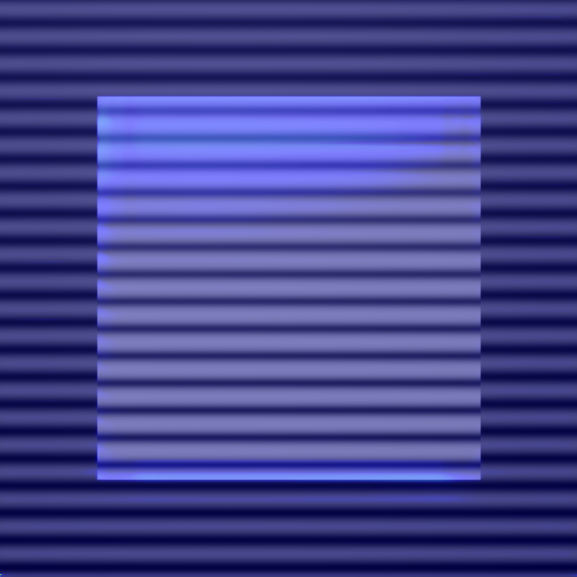}
        \caption{dec4}
    \end{subfigure}
    \caption{GradCAM analysis on the flat plane test. Encoder layers (enc3, enc4) activate selectively on the fringes inside the rectangular front plane, ignoring those on the surrounding background. Decoder layers show minimal spatial attention. The network uses the rectangular boundary to define a region of interest but cannot translate the fringe pattern within it into a uniform depth estimate.}
    \label{fig:gradcam_flat_plane}
\end{figure*}

This failure mode directly confirms the shape prior hypothesis: the network implements a two-stage pipeline of (1)~edge detection to identify object boundaries and define a region of interest, followed by (2)~shape-based depth filling within that region. When presented with a flat surface lacking internal geometric structure, the network correctly detects the boundary and engages encoder attention with the fringes inside it, but the depth-filling stage has no learned shape template to map onto and the prediction collapses to background. The fringe pattern's depth information is therefore reaching the encoder but not being translated into uniform depth, despite that information being present and physically sufficient.

\subsection{Implications for Single-Shot FPP}

These three analyses, linear probing, GradCAM attention mapping, and out-of-distribution testing, converge on a consistent mechanistic explanation: the UNet learns geometric shortcuts rather than fringe-to-depth physics. Each technique provides complementary evidence:

\begin{itemize}
    \item \textbf{Linear probing} demonstrates that geometric information (edges) is encoded 2.82$\times$ more explicitly than depth values in the network's internal representations, indicating the network prioritizes shape over physical quantities.
    \item \textbf{GradCAM} reveals that during prediction, the network attends 1.28$\times$ more strongly to geometric boundaries than to fringe intensity patterns, suggesting decisions are based on shape recognition rather than fringe analysis.
    \item \textbf{Flat plane testing} shows that when presented with a flat surface lacking internal shape features, the encoder selectively attends to the fringes inside the bounded region but the network still collapses to a near-zero (background) prediction across the surface, indicating it cannot translate fringe information into uniform depth without a learned shape template to fill in.
\end{itemize}

This mechanistic understanding explains why reconstruction error (14.54~mm, 18\% of object depth range) remains high despite extensive architectural optimization: the network predicts depth by detecting object boundaries and filling with learned shape templates, not by extracting phase information from fringe patterns. The resulting predictions capture coarse object geometry but lack the fine-grained precision achievable through physics-based phase demodulation. Read against Section~\ref{sec:singleshot_theory}, the diagnosis takes a sharper form: Proposition~\ref{prop:noninjective} shows that the mapping the network was asked to learn does not uniquely exist without fringe-order information, and the shortcut documented here is the specific way an optimizer trained on $\ell_1$/$\ell_2$ regression resolves that non-injectivity (by collapsing onto whatever shape-prior surrogate the training distribution makes available).

These findings have broader implications for the application of deep learning to optical metrology. Neural networks are powerful pattern recognizers, but they exploit whatever statistical regularities lead to low training loss, not necessarily the physical principles we intend them to learn. In our case, correlations between object shapes and their depth profiles provide an easier learning signal than the mathematically precise but spatially distributed fringe-to-depth relationship.

Our results motivate hybrid approaches that explicitly incorporate physical constraints: either through physics-informed loss functions that penalize violations of fringe optics, network architectures that include differentiable phase extraction modules, or training procedures that encourage fringe pattern utilization. Simply providing more data or larger models is unlikely to overcome the fundamental tendency toward geometric shortcut learning demonstrated here. The diagnosis points specifically to the second of these: an architecture whose output space is wrapped phase rather than depth and whose only path from that output to depth runs through a fixed calibration layer, so that the shape-prior solution is removed from the hypothesis space by construction rather than discouraged by a penalty. Designing and verifying such an architecture is the subject of follow-on work.

\section{Conclusion}
\label{sec:conclusion}

Building on FPP-ML-Bench~\citep{haroon2026fppml}, an open photorealistic synthetic benchmark for learning-based single-shot fringe projection profilometry, we asked why the best single-shot baseline plateaus at $\sim 15$~mm object MAE and answered the question mechanistically. Long range is the operating regime rather than the scope of the finding: the shortcut-learning failure we diagnose is a property of single-shot FPP in general, and long range is where it is most visible because the alternative phase-decoding solution is also hardest there.

\noindent\textbf{Theoretical limit.} We formalized the single-shot ambiguity (Section~\ref{sec:singleshot_theory}): the fringe-to-depth mapping is non-injective without fringe-order information, and the depth error induced by an incorrect fringe order grows approximately as $Z^2$ in the working distance. This accounts on theoretical grounds for why single-shot FPP becomes increasingly ill-posed beyond roughly one meter, and it predicts that end-to-end depth regression from a single fringe cannot be well-posed without additional phase or fringe-order information.

\noindent\textbf{Baseline and ablations.} Systematic ablations isolate the design choices that drive single-shot performance. Individual depth normalization improves object MAE 9.1$\times$ over raw depth (16.20 vs.\ 148.07~mm) by decoupling object shape from absolute scale. Removing background fringes degrades performance 2.8--7.3$\times$ across normalizations, establishing that background fringe content is signal rather than noise; gradient-level masking schemes inherit the same failure mode for the same reason, because both physical removal and gradient masking eliminate the spatial phase reference the network needs. Expanding the training set with 18 phase-shifted frames from a single fringe orientation reduces object MAE by 20--24\% relative to the single-frame baseline, while mixing horizontal and vertical orientations or including non-sinusoidal patterns degrades performance, indicating that phase diversity within one orientation is useful augmentation but orientation diversity introduces conflicting gradients. Among six L1/L2-based losses, Hybrid L1 at $\alpha=0.7$ achieves the best object accuracy (14.54~mm); masked losses without a global regularization term fail catastrophically due to scale drift. UNet outperforms TransUNet by 29\%, ResUNet by 60\%, and Pix2Pix by 91\%; the 1.9$\times$ spread across architectures combined with the 14.54~mm residual (18\% of the 80~mm object depth range) points to a representational rather than a capacity-bound limit. Pix2Pix's specific failure, visually plausible predictions with 20--40~mm systematic offsets, reveals that adversarial training optimizes for perceptual plausibility rather than point-wise metrological accuracy.

\noindent\textbf{Mechanistic diagnosis.} Three complementary mechanistic interpretability probes, to our knowledge the first MI study of an FPP network, converge on the same finding: the best UNet baseline does not solve the task by decoding phase from fringe intensity but by detecting object boundaries and filling in learned shape templates. Linear probing of intermediate features shows edges are 2.82$\times$ more linearly readable than depth, with the asymmetry tightest in the decoder where the network actually produces its output. Grad-CAM attribution shows the network attends 1.28$\times$ more strongly to object-boundary pixels than to surrounding fringe texture. An out-of-distribution probe presenting a flat featureless plane at 1.8~m within the trained depth range produces non-uniform, near-zero depth predictions across the surface: encoder attention engages with the fringes inside the rectangular bounded region but the network cannot translate that fringe information into uniform depth without a shape template to fill in. This residual is not a fitting problem the optimizer has not yet solved; it is the consequence of a hypothesis space that contains a lower-loss shortcut, and additional data or larger models will not close it because they do not change the hypothesis space the optimizer searches.

\noindent\textbf{Toward a repair.} The diagnosis has a direct architectural implication. Because the shortcut is available whenever the network outputs depth directly, discouraging it through a loss penalty is unlikely to suffice; the shape-prior solution has to be removed from the hypothesis space by construction. Routing reconstruction through wrapped phase and fringe order, so that depth is a deterministic function of a predicted phase field rather than a directly learned quantity, is the intervention the analysis points to.

\noindent\textbf{Other directions.} Beyond the architectural repair, several directions follow from the evidence above. Phase-guided learning that supplies wrapped or unwrapped phase as input or intermediate supervision is a direct intervention against the boundary-shortcut diagnosis. Learned components restricted to post-processing of traditional phase reconstructions (denoising, hole-filling, outlier removal) are consistent with what the networks in our architecture comparison actually do well. Multi-view fusion exploiting the 6 viewpoints per object available in FPP-ML-Bench provides genuine additional depth information rather than additional shape-prior evidence. Sim-to-real transfer via domain adaptation and domain randomization~\citep{tobin2017domain} from the VIRTUS-FPP framework to physical hardware would extend the methodology to challenging materials (specular, translucent) and lighting conditions. The shape-prior shortcut diagnosed in Section~\ref{sec:interpretability} is a general property of single-shot deep learning on FPP geometry rather than an artifact of the simulation pathway: it is the same phenomenon Geirhos et al.~\citep{geirhos2020shortcut} documented across vision, and it inherits from the optimizer exploiting the available hypothesis space. VIRTUS-FPP's physical validation concerns the forward model, that the simulator reproduces the fringe formation and reconstruction behavior of a real FPP system, which is what makes the reported error magnitudes physically meaningful; it does not by itself establish that a network trained only on synthetic data transfers to hardware without adaptation, which we do not claim and leave to the sim-to-real work noted above. The diagnosis itself, being a property of the single-shot information deficit and of FPP geometry, would arise on any single-shot FPP setup where shape priors are exploitable.

\noindent We adopt FPP-ML-Bench as a shared baseline rather than as a verdict. Prior single-shot FPP work has been evaluated largely on bespoke datasets with per-paper choices of normalization, loss, evaluation region, and error statistic, which makes cross-study comparison unreliable and hides which design decisions actually matter. The dataset, the object/background/overall evaluation protocol, and the explicit reporting conventions for normalization, loss, masking, and training-set composition together define a reproducible and transparent evaluation surface for single-shot FPP. Our specific findings (individual normalization, background fringes as signal, single-orientation multi-frame gains, Hybrid L1 at $\alpha=0.7$, UNet over the larger transformer and residual variants, and the mechanistic evidence that the baseline decodes shape rather than phase) are offered as falsifiable baselines on that shared surface, to be confirmed, refined, or overturned as the community reports results against the same protocol.


\section*{Declaration of Competing Interest}
The authors declare that they have no known competing financial interests or personal relationships that could have appeared to influence the work reported in this paper.

\section*{Code, Data, and Materials Availability}
The FPP-ML-Bench dataset used in this work is publicly available at \linkable{https://huggingface.co/datasets/aharoon/fpp-ml-bench}. Code will be released upon acceptance.

\section*{Acknowledgments}
We thank Iowa State University for access to computational resources.

\printcredits

\bibliographystyle{cas-model2-names}
\bibliography{report}

@book{zhang2016high,
  title={High-Speed 3D Imaging with Digital Fringe Projection Techniques},
  author={Zhang, Song},
  year={2016},
  publisher={CRC Press},
  edition={1st},
  doi={10.1201/b19565}
}

@article{geng2011structured,
  title={Structured-light 3D surface imaging: a tutorial},
  author={Geng, Jason},
  journal={Advances in optics and photonics},
  volume={3},
  number={2},
  pages={128--160},
  year={2011},
  publisher={Optical Society of America}
}

@article{zhang2010recent,
  title={Recent progresses on real-time 3D shape measurement using digital fringe projection techniques},
  author={Zhang, Song},
  journal={Optics and lasers in engineering},
  volume={48},
  number={2},
  pages={149--158},
  year={2010},
  publisher={Elsevier}
}

@article{van2019deep,
  title={Deep neural networks for single shot structured light profilometry},
  author={Van der Jeught, Sam and Dirckx, Joris JJ},
  journal={Optics express},
  volume={27},
  number={12},
  pages={17091--17101},
  year={2019},
  publisher={Optical Society of America}
}

@article{zuo2022deep,
  title={Deep learning in optical metrology: a review},
  author={Zuo, Chao and Qian, Jiaming and Feng, Shijie and Yin, Wei and Li, Yixuan and Fan, Pengfei and Han, Jing and Qian, Kemao and Chen, Qian},
  journal={Light: Science \& Applications},
  volume={11},
  number={1},
  pages={39},
  year={2022},
  publisher={Nature Publishing Group UK London}
}

@article{zhu2022hformer,
  title={Hformer: Hybrid convolutional neural network transformer network for fringe order prediction in phase unwrapping of fringe projection},
  author={Zhu, Xinjun and Han, Zhiqiang and Yuan, Mengkai and Guo, Qinghua and Wang, Hongyi and Song, Limei},
  journal={Optical Engineering},
  volume={61},
  number={9},
  pages={093107--093107},
  year={2022},
  publisher={Society of Photo-Optical Instrumentation Engineers}
}

@article{wang2021single,
  title={Single-shot fringe projection profilometry based on deep learning and computer graphics},
  author={Wang, Fanzhou and Wang, Chenxing and Guan, Qingze},
  journal={Optics Express},
  volume={29},
  number={6},
  pages={8024--8040},
  year={2021},
  publisher={Optical Society of America}
}

@article{nguyen2020single,
  title={Single-shot 3D shape reconstruction using structured light and deep convolutional neural networks},
  author={Nguyen, Hieu and Wang, Yuzeng and Wang, Zhaoyang},
  journal={Sensors},
  volume={20},
  number={13},
  pages={3718},
  year={2020},
  publisher={MDPI}
}

@article{li2025deep,
  title={Deep-learning-enabled single-shot fringe projection profilometry based on inner shifting-phase encoding},
  author={Li, Jinlong and Zhang, Kuo and Luo, Lin and Liu, Gaokun and Tang, Tao and Wang, Zhijie and Wan, Yingying},
  journal={Optics Express},
  volume={33},
  number={23},
  pages={49530--49550},
  year={2025},
  publisher={Optica Publishing Group}
}

@article{wang2025end,
  title={End-to-end single-shot composite color FPP network for multiple separated objects reconstruction},
  author={Wang, Lianpo and Xing, Yanyang},
  journal={Measurement},
  volume={246},
  pages={116697},
  year={2025},
  publisher={Elsevier}
}

@article{zheng2020fringe,
  author={Y. Zheng and S. Wang and Q. Li and B. Li},
  title={Fringe projection profilometry by conducting deep learning from its digital twin},
  journal={Optics Express},
  volume={28},
  number={24},
  pages={36568--36583},
  year={2020},
  publisher={Optica Publishing Group}
}

@article{ueda2021fringe,
  title={Fringe projection profilometry system verification for 3D shape measurement using virtual space of game engine},
  author={Ueda, Kazumasa and Ikeda, Kanami and Koyama, Osanori and Yamada, Makoto},
  journal={Optical Review},
  volume={28},
  number={6},
  pages={723--729},
  year={2021},
  publisher={Springer}
}

@article{zhang2023measurement,
  title={Measurement Simulation System of Fringe Projection Profilometry Based on Ray Tracing},
  author={Zhang, Qiushuang and Xing, Mingyi and Li, Hongbin and Li, Xu and Wang, Tingli},
  journal={IEEE Access},
  volume={11},
  pages={89616--89624},
  year={2023},
  publisher={IEEE}
}

@article{ikeda2025deep,
  title={Deep-learning-assisted single-shot 3D shape and color measurement using color fringe projection profilometry},
  author={Ikeda, Kanami and Usuki, Takahiro and Kurita, Yumi and Matsueda, Yuya and Koyama, Osanori and Yamada, Makoto},
  journal={Optical Review},
  pages={1--12},
  year={2025},
  publisher={Springer}
}

@ARTICLE{HaroonVIRTUS2025,
  author={Haroon, Adam and Lakshman, Anush and Balasubramaniam, Badrinath and Li, Beiwen},
  journal={IEEE Sensors Journal}, 
  title={VIRTUS-FPP: Virtual Sensor Modeling for Fringe Projection Profilometry in NVIDIA Isaac Sim}, 
  year={2026},
  volume={},
  number={},
  pages={1-1},
  doi={10.1109/JSEN.2026.3698278}
}

@article{qian2021high,
  title={High-resolution real-time 360\textdegree{} 3D surface defect inspection with fringe projection profilometry},
  author={Qian, Jiaming and Feng, Shijie and Xu, Mingzhu and Tao, Tianyang and Shang, Yuhao and Chen, Qian and Zuo, Chao},
  journal={Optics and Lasers in Engineering},
  volume={137},
  pages={106382},
  year={2021},
  publisher={Elsevier}
}

@inproceedings{deng2016three,
  title={Three-dimensional surface inspection for semiconductor components with fringe projection profilometry},
  author={Deng, Fuqin and Ding, Yi and Peng, Kai and Xi, Jiangtao and Yin, Yongkai and Zhu, Ziqi},
  booktitle={Optical Metrology and Inspection for Industrial Applications IV},
  volume={10023},
  pages={175--186},
  year={2016},
  organization={SPIE}
}

@article{zhang2023machine,
  title={Machine learning enhanced high dynamic range fringe projection profilometry for in-situ layer-wise surface topography measurement during LPBF additive manufacturing},
  author={Zhang, Haolin and Vallabh, Chaitanya Krishna Prasad and Zhao, Xiayun},
  journal={Precision Engineering},
  volume={84},
  pages={1--14},
  year={2023},
  publisher={Elsevier}
}

@article{zhang2022systematic,
  title={A systematic study and framework of fringe projection profilometry with improved measurement performance for in-situ LPBF process monitoring},
  author={Zhang, Haolin and Vallabh, Chaitanya Krishna Prasad and Xiong, Yubo and Zhao, Xiayun},
  journal={Measurement},
  volume={191},
  pages={110796},
  year={2022},
  publisher={Elsevier}
}

@inproceedings{haroon2024autonomous,
  title={Autonomous robotic 3D scanning for smart factory planning},
  author={Haroon, Adam and Lakshman, Anush and Mundy, Micah and Li, Beiwen},
  booktitle={Dimensional Optical Metrology and Inspection for Practical Applications XIII},
  volume={13038},
  pages={110--118},
  year={2024},
  organization={SPIE}
}

@article{wang2024robotic,
  title={Robotic measurement system based on cooperative optical profiler integrating fringe projection with photometric stereo for highly reflective workpiece},
  author={Wang, Xi and Shen, Yijun and Jian, Zhenxiong and Wen, Daizhou and Zhang, Xinquan and Zhu, LiMin and Ren, Mingjun},
  journal={Robotics and Computer-Integrated Manufacturing},
  volume={88},
  pages={102739},
  year={2024},
  publisher={Elsevier}
}

@inproceedings{Ronneberger2015,
  title={{U-Net}: Convolutional Networks for Biomedical Image Segmentation},
  author={Ronneberger, Olaf and Fischer, Philipp and Brox, Thomas},
  booktitle={Medical Image Computing and Computer-Assisted Intervention (MICCAI)},
  pages={234--241},
  year={2015},
  publisher={Springer}
}

@inproceedings{Isola2017,
  title={Image-to-Image Translation with Conditional Adversarial Networks},
  author={Isola, Phillip and Zhu, Jun-Yan and Zhou, Tinghui and Efros, Alexei A},
  booktitle={Proceedings of the IEEE Conference on Computer Vision and Pattern Recognition (CVPR)},
  pages={1125--1134},
  year={2017}
}

@inproceedings{Wang2018,
  title={High-Resolution Image Synthesis and Semantic Manipulation with Conditional {GAN}s},
  author={Wang, Ting-Chun and Liu, Ming-Yu and Zhu, Jun-Yan and Tao, Andrew and Kautz, Jan and Catanzaro, Bryan},
  booktitle={Proceedings of the IEEE Conference on Computer Vision and Pattern Recognition (CVPR)},
  pages={8798--8807},
  year={2018}
}

@inproceedings{balasubramaniam2023single,
  title={Single Shot 3D Shape Measurement of Non-Volatile Data Storage Devices},
  author={Balasubramaniam, Badrinath and Li, Beiwen},
  booktitle={International Manufacturing Science and Engineering Conference},
  volume={87240},
  pages={V002T06A010},
  year={2023},
  organization={American Society of Mechanical Engineers}
}

@inproceedings{deng2009imagenet,
  title={{ImageNet}: A large-scale hierarchical image database},
  author={Deng, Jia and Dong, Wei and Socher, Richard and Li, Li-Jia and Li, Kai and Fei-Fei, Li},
  booktitle={2009 IEEE conference on computer vision and pattern recognition},
  pages={248--255},
  year={2009},
  organization={IEEE}
}

@inproceedings{lin2014microsoft,
  title={Microsoft {COCO}: Common objects in context},
  author={Lin, Tsung-Yi and Maire, Michael and Belongie, Serge and Hays, James and Perona, Pietro and Ramanan, Deva and Doll{\'a}r, Piotr and Zitnick, C Lawrence},
  booktitle={Computer Vision--ECCV 2014: 13th European Conference},
  pages={740--755},
  year={2014},
  organization={Springer}
}

@article{calli2017yale,
  title={Yale-CMU-Berkeley dataset for robotic manipulation research},
  author={Calli, Berk and Singh, Arjun and Bruce, James and Walsman, Aaron and Konolige, Kurt and Srinivasa, Siddhartha and Abbeel, Pieter and Dollar, Aaron M},
  journal={The International Journal of Robotics Research},
  volume={36},
  number={3},
  pages={261--268},
  year={2017},
  publisher={SAGE Publications Sage UK: London, England}
}

@misc{Nvidia2025PhysicalAIWarehouse,
  title = {{Physical AI Spatial Intelligence Warehouse Dataset}},
  author = {{NVIDIA}},
  year = {2025},
  howpublished = {\url{https://huggingface.co/datasets/nvidia/PhysicalAI-Spatial-Intelligence-Warehouse}},
  note = {Accessed: 2026-01-12},
}

@article{sansoni1999three,
  title={Three-dimensional vision based on a combination of Gray-code and phase-shift light projection: analysis and compensation of the systematic errors},
  author={Sansoni, Giovanna and Trebeschi, Marco and Docchio, Franco},
  journal={Applied optics},
  volume={38},
  number={31},
  pages={6565--6573},
  year={1999},
  publisher={Optica Publishing Group}
}

@article{feng2021calibration,
  title={Calibration of fringe projection profilometry: A comparative review},
  author={Feng, Shijie and Zuo, Chao and Zhang, Liang and Tao, Tianyang and Hu, Yan and Yin, Wei and Qian, Jiaming and Chen, Qian},
  journal={Optics and lasers in engineering},
  volume={143},
  pages={106622},
  year={2021},
  publisher={Elsevier}
}

@article{ma16155443,
  author={Zapico, Pablo and Meana, Victor and Cuesta, Eduardo and Mateos, Sabino},
  title={Optical Characterization of Materials for Precision Reference Spheres for Use with Structured Light Sensors},
  journal={Materials},
  volume={16},
  number={15},
  pages={5443},
  year={2023},
  doi={10.3390/ma16155443}
}

@article{mi13101607,
  author={Ou, Jia and Xu, Tingfa and Gan, Xiaochuan and He, Xuejun and Li, Yan and Qu, Jiansu and Zhang, Wei and Cai, Cunliang},
  title={Comparative Analysis on the Effect of Surface Reflectance for Laser {3D} Scanner Calibrator},
  journal={Micromachines},
  volume={13},
  number={10},
  pages={1607},
  year={2022},
  doi={10.3390/mi13101607}
}

@inproceedings{tobin2017domain,
  title={Domain randomization for transferring deep neural networks from simulation to the real world},
  author={Tobin, Josh and Fong, Rachel and Ray, Alex and Schneider, Jonas and Zaremba, Wojciech and Abbeel, Pieter},
  booktitle={2017 IEEE/RSJ international conference on intelligent robots and systems (IROS)},
  pages={23--30},
  year={2017},
  organization={IEEE}
}

@article{chen2021transunet,
  title={Transunet: Transformers make strong encoders for medical image segmentation},
  author={Chen, Jieneng and Lu, Yongyi and Yu, Qihang and Luo, Xiangde and Adeli, Ehsan and Wang, Yan and Lu, Le and Yuille, Alan L and Zhou, Yuyin},
  journal={arXiv preprint arXiv:2102.04306},
  year={2021}
}

@article{olah2020zoom,
  title={Zoom In: An Introduction to Circuits},
  author={Olah, Chris and Cammarata, Nick and Schubert, Ludwig and Goh, Gabriel and Petrov, Michael and Carter, Shan},
  journal={Distill},
  volume={5},
  number={3},
  pages={e00024.001},
  year={2020},
  doi={10.23915/distill.00024.001}
}

@article{alain2016understanding,
  title={Understanding intermediate layers using linear classifier probes},
  author={Alain, Guillaume and Bengio, Yoshua},
  journal={arXiv preprint arXiv:1610.01644},
  year={2016}
}

@inproceedings{selvaraju2017grad,
  title={Grad-{CAM}: Visual explanations from deep networks via gradient-based localization},
  author={Selvaraju, Ramprasaath R. and Cogswell, Michael and Das, Abhishek and Vedantam, Ramakrishna and Parikh, Devi and Batra, Dhruv},
  booktitle={Proceedings of the IEEE International Conference on Computer Vision (ICCV)},
  pages={618--626},
  year={2017}
}

@article{geirhos2020shortcut,
  title={Shortcut learning in deep neural networks},
  author={Geirhos, Robert and Jacobsen, J{\"o}rn-Henrik and Michaelis, Claudio and Zemel, Richard and Brendel, Wieland and Bethge, Matthias and Wichmann, Felix A.},
  journal={Nature Machine Intelligence},
  volume={2},
  number={11},
  pages={665--673},
  year={2020},
  publisher={Nature Publishing Group}
}

@article{evans2024benchmark,
  title={Deep learning for single-shot structured light profilometry: A comprehensive dataset and performance analysis},
  author={Evans, Rhys G and Devlieghere, Ester and Keijzer, Robrecht and Dirckx, Joris JJ and Van der Jeught, Sam},
  journal={Journal of Imaging},
  volume={10},
  number={8},
  pages={179},
  year={2024},
  publisher={MDPI}
}

@inproceedings{haroon2026fppml,
  author = {Anush Lakshman S. and Adam Haroon and Beiwen Li},
  title = {{Comprehensive machine learning benchmarking for fringe projection profilometry with photorealistic synthetic data}},
  volume = {13904},
  booktitle = {Photonic Instrumentation Engineering XIII},
  editor = {Lynda E. Busse and Yakov Soskind},
  organization = {International Society for Optics and Photonics},
  publisher = {SPIE},
  pages = {1390402},
  year = {2026},
  doi = {10.1117/12.3082257},
  URL = {https://doi.org/10.1117/12.3082257}
}

@article{zuo2025deep,
  title={Deep-learning-based endoscopic single-shot fringe projection profilometry},
  author={Zuo, Ruizhi and Wei, Shuwen and Wang, Yaning and Huang, Ruichen and Rodgers, Wayne Wonseok and Yu, Jinglun and Hsieh, Michael H and Krieger, Axel and Kang, Jin U},
  journal={Journal of biomedical optics},
  volume={30},
  number={8},
  pages={086003--086003},
  year={2025},
  publisher={Society of Photo-Optical Instrumentation Engineers}
}

\end{document}